\documentclass[screen]{jair}

\AtBeginDocument{%
  \fancyfoot{}%
  \fancyfoot[C]{\thepage}% keep only the page number (optional)
  \fancypagestyle{firstpagestyle}{\fancyhf{}\fancyfoot[C]{\thepage}}%
}

\usepackage[firstpageonly=false]{draftwatermark}
\SetWatermarkText{Under review}
\SetWatermarkScale{1.6}
\SetWatermarkLightness{0.92}

\JAIRAE{Insert JAIR AE Name}
\JAIRTrack{} % Insert JAIR Track Name only if part of a special track
\RequirePackage[
  datamodel=acmdatamodel,
  style=acmauthoryear,
  backend=biber,
  giveninits=true,
  uniquename=init
  ]{biblatex}
\usepackage[utf8]{inputenc}
\usepackage{tikz}
\usetikzlibrary{arrows.meta,positioning,shapes.geometric,fit,backgrounds}
\usepackage{algorithm}
\usepackage{algorithmic}
\usepackage{mdframed}
\usepackage{xcolor}
\usepackage{amsthm}
\usepackage{enumerate}
\usepackage{booktabs}
\usepackage[shortlabels]{enumitem}
\newtheorem{definition}{Definition}
\newtheorem{proposition}{Proposition}
\newtheorem{remark}{Remark}

\usepackage[utf8]{inputenc}
\DeclareUnicodeCharacter{0308}{\"{}}

\begin{document}

%\title{JAIR Example Template}
\title[Fairness of Explanations]{Fairness of Explanations in Artificial Intelligence (AI): A Unifying Framework, Axioms,  and Future Direction toward Responsible AI}

%% used to denote shared contribution to the research and/or corresponding author.
\author{Gideon Popoola}
\authornote{Corresponding Author.}
\orcid{0009-0001-9596-8115}
\email{gideon.popoola@student.montana.edu}
\affiliation{%
  \institution{Montana State University}
  \city{Bozeman}
  \state{Montana}
  \country{USA}
}

\author{John Sheppard}
\orcid{0000-0001-9487-5622}
\email{john.sheppard@montana.edu}
\affiliation{%
  \institution{Montana State University}
  \city{Bozeman}
  \state{Montana}
  \country{USA}}

%% The short list of authors must be made of the list of all authors' lastnames.
\renewcommand{\shortauthors}{Gideon and John}
%% If this is too long and overlaps other information printed in the page headers, use
%\renewcommand{\shortauthors}{Xu et al.}

\begin{abstract}
Machine learning algorithms are being used in high-stakes decisions, including those in criminal justice, healthcare, credit, and employment. The research community has responded with two largely independent research fields: \emph{algorithmic fairness}, which targets equitable outcomes, and \emph{explainable AI} (XAI), which targets interpretable reasoning. This survey identifies and maps a novel blind spot at their intersection, which is a model that can satisfy every standard fairness criterion in its outputs while being profoundly unfair in its \emph{reasoning process}. We refer to this as the procedural bias, and mitigating it requires treating the fairness of explanations as a distinct object of scientific study.

To our knowledge, we provide the first unified theoretical and literature review of this emerging field and elucidate the drawbacks of post-hoc explainers in certifying explanation fairness. Our central contribution is a \emph{conditional invariance framework} formalizing explanation fairness as the requirement that explanations should be indifferent regardless of the protected attributes $ P(E(X) \in \cdot \mid X_\text{rel} = x_\text{rel},\, A = a) = P(E(X) \in \cdot \mid X_\text{rel} = x_\text{rel},\, A = b)$
for all task-relevant $x$, a single principle from which all existing explanation fairness metrics emerge as partial operationalizations. We introduce a seven-dimensional taxonomy in Table \ref{tab:taxonomy}, identify three generative mechanisms of explanation inequity (representation-driven, explanation-model mismatch, actionability-driven), and propose a canonical six-step evaluation workflow for operationalizing explanation fairness audits in practice.

Beyond the taxonomy, we formalize and synthesize five results that, while partially anticipated by individual papers, have not previously been stated as standalone claims: (i) explanation fairness is an interventional quantity not identifiable from observational data without causal assumptions, generalizing the specific vulnerability demonstrations in the fairwashing literature to a structural identifiability barrier for all model-agnostic post-hoc methods; (ii) the fairness-accuracy-explainability trilemma reflects deep structural tensions, not merely engineering trade-offs; (iii) attribution parity is necessary but not sufficient for \emph{epistemic fairness}, the equal ability of all groups to understand and act on explanations; (iv) five formal \emph{axioms of fair explanation systems} unify the positive design requirements the literature has only implied; and (v) a research agenda of concrete imperatives, not merely open questions, that specifies what must be built for the field to mature. We survey over 300 publications spanning 2016--2025, provide a failure taxonomy with remediation guidance, a practitioner decision guide, and seven prioritized open research directions. This survey is intended as a structured theoretical foundation for the emergent field of explanation fairness.
\end{abstract}

%\begin{abstract}
%      A clear and well-documented \LaTeX\ document is presented as an
%  article formatted for publication by ACM in a conference proceedings
%  or journal publication. Based on the ``acmart'' document class, this
%  article presents and explains many of the common variations, as well
%  as many of the formatting elements an author may use in the
%  preparation of the documentation of their work.
%\end{abstract}

%% JAIR Note: 
%% Do not include ACM CCS Concepts or Keywords

%% To be updated by authors.
\received{20 February 2007}
\received[accepted]{5 June 2009}

\maketitle

\section{Introduction}
\label{sec:intro}

Machine learning (ML) algorithms have become ubiquitous in day-to-day decision making, especially social sensitive domains that affect humans daily. ML models now determine who receives bail \cite{angwin2016machine, morin2024machine}, which patients receive additional medical interventions \cite{obermeyer2019dissecting, bravi2024development, ibrahim2021role}, whether a loan application is approved \cite{fuster2022predictably, sheikh2020approach, chouksey2023machine}, who is shortlisted for employment \cite{dastin2018amazon}, and which social media content individuals are exposed to \cite{bakshy2015exposure, shu2018content}. The remarkable predictive power of ML algorithms, such as deep neural networks, gradient-boosted trees, and large language models, arises from their ability to discover high-order patterns in large and diverse datasets \cite{huang2024unlocking, xiong2020evaluating}. However, this very capacity can introduce a profound opacity in which the internal computations that produce a prediction are often too complex to be understood by the individuals they affect, the practitioners who deploy them, or the regulators mandated to oversee them \cite{castelvecchi2016can, fernandez2024opacity}. The black-box nature of ML is therefore an important problem to address to achieve responsible artificial intelligence (AI) use, especially in socially sensitive domains.

The opacity of black-box models has generated two intensively studied research communities. The first is \emph{algorithmic fairness}, which investigates whether automated systems produce outcomes that are equitable across demographic groups defined by race, gender, age, disability, national origin, or other protected characteristics. Foundational work by \cite{dwork2012fairness}, \cite{kusner2017counterfactual}, \cite{hardt2016equality}, and \cite{chouldechova2017fair} has produced a rich taxonomy of fairness criteria, including statistical parity, equalized odds, equal opportunity, individual fairness, and counterfactual fairness. The second research community is \emph{explainable artificial intelligence} (XAI), which develops methods for making model behavior interpretable, transparent, and accountable. Influential methods include SHAP \cite{lundberg2017unified}, LIME \cite{ribeiro2016should}, GradCAM \cite{selvaraju2017grad}, ANCHOR\cite{ribeiro2018anchors}, Integrated Gradient \cite{sundararajan2017axiomatic}, and counterfactual explanation generators \cite{wachter2017counterfactual,mothilal2020explaining}.

\begin{table*}[t]
\caption{Unified taxonomy of fairness-of-explanation methods. Each row represents one of the seven survey dimensions with representative methods, explanation types, and fairness notions.}
\label{tab:taxonomy}
\small
\begin{tabular}{p{0.5cm} p{2.8cm} p{3.8cm} p{3.0cm} p{2.8cm} p{1.4cm}}
\toprule
\textbf{\#} & \textbf{Dimension} & \textbf{Representative Methods} & \textbf{Explanation Type} & \textbf{Fairness Notion} & \textbf{Sec.} \\
\midrule
1 & Distributional fairness of feature attributions &
  Slack et al., Aïvodji et al., Wang \& Wu, Begley et al. &
  SHAP / LIME / gradients &
  Equalized explainability; fairwashing &
  \ref{sec:attribution_fairness} \\
\addlinespace[2pt]
2 & Procedural fairness and explanation difference &
  CFA (\(\Delta\)REF, \(\Delta\)VEF), GPF\(_\text{FAE}\), EDiff &
  Feature attribution (SHAP) &
  Group procedural fairness &
  \ref{sec:procedural} \\
\addlinespace[2pt]
3 & Multi-objective optimization for joint fairness &
  EDiff loss, AnyLoss, FairMOE, Agarwal reductions &
  Any &
  Utility--fairness--procedural Pareto &
  \ref{sec:multiobj} \\
\addlinespace[2pt]
4 & Fairness of counterfactual explanations / recourse &
  DiCE, Wachter, CERTIFAI, FACTS, Causal recourse &
  Counterfactual &
  Recourse cost / access parity &
  \ref{sec:cf_recourse} \\
\addlinespace[2pt]
5 & Graph-structured explanation fairness &
  REFEREE, BIND, GNNUERS, CF-GNNExplainer &
  Subgraph / node attribution &
  Node-level, consumer fairness &
  \ref{sec:graph} \\
\addlinespace[2pt]
6 & Intersectional fairness and gerrymandering &
  Kearns et al., Foulds et al., Hébert-Johnson et al. &
  Any &
  Subgroup / differential fairness &
  \ref{sec:intersectional} \\
\addlinespace[2pt]
7 & Beyond attributions: rules, examples, concepts &
  ANCHOR, TCAV, CBMs, LTNs, ProtoDash &
  Rules / concepts / neurosymbolic &
  Procedural; conceptual bias &
  \ref{sec:beyond} \\
\bottomrule
\end{tabular}
\end{table*}

Despite the independent maturity of these fields, their \emph{intersection}, the fairness of explanations as a distinct research object, has received comparatively limited research attention \cite{dai2022fairness}. The common assumption in much XAI research is that explanation methods provide uniform informational value to all groups, while the common assumption in much fairness research is that bias resides in the model's predicted outcomes rather than in the reasoning process leading to them \cite{moons2019probast, grabowicz2022marrying}. Both assumptions are now empirically and theoretically challenged. Several research works have demonstrated that SHAP and LIME can be deliberately deceived to conceal discriminatory behavior behind apparently fair explanations \cite{slack2020fooling, bordt2022post}. Also,   \cite{aivodji2019fairwashing} formalized \emph{fairwashing} as the construction of a proxy model that appears fair to an explanation-based auditor, \cite{anders2020fairwashing} proved an impossibility result showing that no axiomatic post-hoc explanation method can reliably distinguish fair from unfair models, and ~\cite{begley2020explainability} documented that explanation quality is systematically lower for demographic minorities, exacerbating inequities in high-stakes decision settings. All these observations point to a deeper problem in ML prediction, which is procedural bias.
Procedural bias occurs when a model uses different reasoning processes for different demographic groups solely because of their protected attribute(s). Formally, procedural bias can also be referred to as the (Un)fairness of explanation.

An example of unfairness (bias) in the explanation of a model is depicted in Figure \ref{fig:ediff_illustration}. Two loan applicants who are identical in all features except their protected group may both receive positive credit decisions, which satisfies standard outcome-oriented fairness criteria (e.g., statistical parity), however, the model may explain these decisions in systematically different terms: penalizing the Group 1 applicant more severely for accumulated debt while rewarding the Group 0 applicant more generously for income. Such a model creates incentive structures that reinforce societal biases even as it satisfies outcome-level fairness constraints. The model is \emph{procedurally biased} because it evaluates identical individuals differently depending on their protected-group membership. This observation motivates a fundamental reconceptualization of fairness as a property that must encompass not only \emph{what} a model predicts but also \emph{how} it explains those predictions.

\begin{figure}[]
  \centering
  \includegraphics[width=10cm]{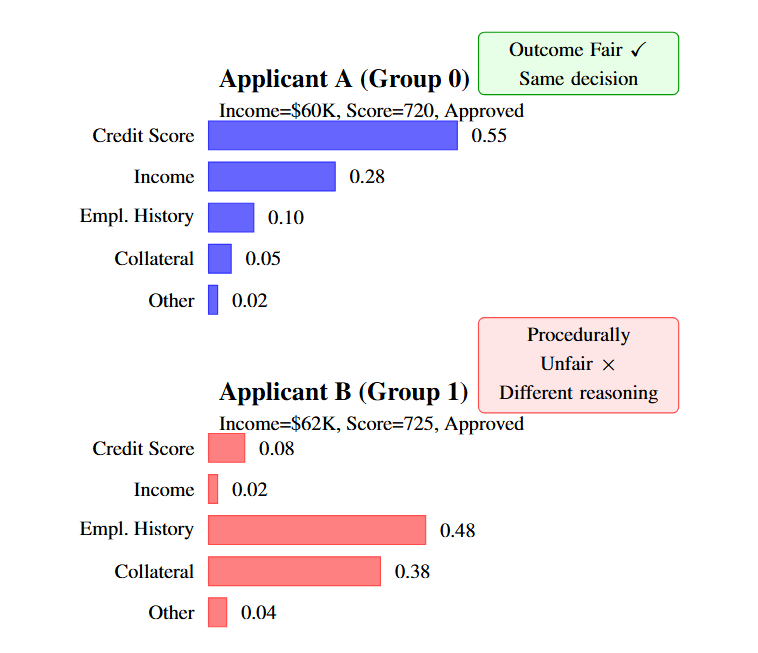}
  \caption{Hypothetical illustration of procedural unfairness: two hypothetical loan applicants receive identical outcomes, but their model explanations differ substantially, indicating that the decision criteria are applied asymmetrically across protected groups.}
  \label{fig:ediff_illustration}
\end{figure}

This survey maps the state of research on fairness of explanations in AI, organizing it under the conceptual umbrella of \emph{procedural fairness}, a notion with roots in the social-psychological literature of Thibaut and Walker \cite{thibaut1975procedural} and the political philosophy of Rawls \cite{rawls197theory}, and recently formalized in machine learning by Zhao et al. \cite{zhao2023fairness}, Wang et al.\cite{wang2024procedural}, and Germino et al.~\cite{germino2025ediff}. Procedural fairness demands not merely equitable outcomes (distributional fairness) but an equitable \emph{decision-making process}, which is that the criteria on which individuals are evaluated should not vary systematically across protected groups.

\subsection{Scope and Contributions}
This survey makes the following contributions to the literature:
\begin{itemize}
  \item \textbf{Conditional invariance framework.} We propose a formal unifying principle (Definition~\ref{def:ef}, Eq.~\ref{eq:ef_invariance}) showing that all major explanation fairness metrics are partial operationalizations of a single conditional invariance condition (Table~\ref{tab:unification}).
  \item \textbf{Five axioms of explanation fairness.} We extract five design axioms from the invariance condition and the metrics literature, Process Consistency, Counterfactual Stability, Distributional Parity, Actionability Symmetry, and Epistemic Accessibility, that constitute the positive requirements any fair explanation system must satisfy (Section~\ref{sec:axioms}). Axiom 5 is currently unmeasured by any existing metric, defining the epistemic fairness research frontier.
  \item \textbf{Identifiability proposition and hard limit.} We formalize the identifiability barrier as Proposition~\ref{prop:identifiability} with proof sketch, and state as Remark~\ref{rem:hardlimit} that post-hoc methods are \emph{structurally incapable} of certifying procedural fairness, not as a hedged implication but as a direct claim.
  \item \textbf{Mechanism taxonomy.} We identify three generative pathways,representation-driven inequity, explanation-model mismatch, and actionability-driven inequity (Section~\ref{sec:mechanisms}), and trace the full causal chain from data imbalance to recourse disparity.
  \item \textbf{Standard evaluation workflow.} We propose a canonical six-step audit protocol (Section~\ref{sec:eval_workflow}) designed to probe all three generative pathways, from automated metric-based steps through human interpretability validation.
  \item \textbf{Explicit limits.} We state without hedging five things explanation fairness cannot guarantee (Section~\ref{sec:limits}): post-hoc methods cannot certify fairness; explanation fairness $\neq$ model fairness; post-hoc methods cannot ensure causal neutrality; epistemic inequities are invisible to statistical metrics; the identifiability barrier is permanent.
  \item \textbf{Post-hoc vs.\ in-processing vs.\ intrinsic comparison.} Table~\ref{tab:posthoc_vs_intrinsic} maps three architectural families to the five axioms, making the design implications of each architecture explicit.
  \item \textbf{Research imperatives.} We state five concrete directives for what the field must build (Section~\ref{sec:imperatives}), not open questions but architectural and methodological requirements, and seven open problems prioritized by urgency and tractability.
  \item \textbf{Epistemic fairness.} We introduce the concept of epistemic fairness, the equal ability of all groups to understand and act on explanations, as a distinct dimension not implied by attribution parity (Section~\ref{sec:epistemic}).
  \item \textbf{Seven-dimension taxonomy.} We propose a hierarchical taxonomy spanning feature attribution fairness, procedural metrics, multi-objective optimization, counterfactual fairness, graph explanation fairness, intersectional fairness, and beyond-attribution methods.
  \item \textbf{Failure taxonomy and decision guide.} A structured failure taxonomy (Table~\ref{tab:failure_taxonomy}) and practitioner decision guide (Table~\ref{tab:decision_guide}) for method selection.
  \item \textbf{Cross-domain synthesis.} A four-domain coverage (credit, criminal justice, healthcare, employment) with a cross-domain synthesis identifying structural features that explain fairness frameworks must address in high-stakes deployment.
  \item \textbf{In-processing procedural fairness.} We survey GCIG/FairX~\cite{popoola2026gcig}, the first training-time framework enforcing group counterfactual explanation invariance via a differentiable regularization loss, and GESD/FEU~\cite{popoola2025gesd}, a stability-based procedural metric and multi-objective evolutionary optimization framework.
  \item \textbf{Intersectional procedural fairness.} We survey MESD/UEF~\cite{popoola2025mesd}, the first procedural fairness metric designed explicitly for intersectional subgroups, incorporating label-aware aggregation, empirical-Bayes shrinkage, and CVaR tail weighting to prevent fairness gerrymandering at the explanation level.
  \item \textbf{Comprehensive coverage.} We surveyed over 300 publications from NeurIPS, ICML, ICLR, FAccT, AAAI, AIES, ACM Computing Surveys, IEEE Transactions, Springer, Elsevier, and others, covering 2016--2025.
  %\item \textbf{Mathematical formulations.} Formal definitions for SHAP (Eq.~\ref{eq:shap}), LIME (Eq.~\ref{eq:lime}), $\Delta\text{REF}$ (Eq.~\ref{eq:ref}), $\Delta\text{VEF}$ (Eq.~\ref{eq:vef}), GPF$_\text{FAE}$ (Eq.~\ref{eq:gpf}), EDiff (Eq.~\ref{eq:ediff}), and graph-level fairness metrics.
\end{itemize}

\subsection{Related Surveys}

Prior surveys share a similar pattern that this work corrects. Mehrabi et al.\cite{mehrabi2021survey} comprehensively survey algorithmic bias and fairness but treat explanation only as a tool for diagnosing outcome bias, never as a source of bias in its own right. Guidotti et al.~\cite{guidotti2018survey} and Arrieta et al. \cite{arrieta2020explainable} survey XAI comprehensively while treating fairness as a desideratum rather than an evaluative dimension of the explanation itself. Adadi and Berrada \cite{adadi2018peeking} survey post-hoc XAI without noting that post-hoc methods under black-box access cannot reliably certify the fairness of the model they approximate. Baniecki and Biecek \cite{baniecki2024adversarial} survey adversarial XAI attacks, including fairwashing, but do not draw the implication that fairwashing is not a pathological attack but a \emph{structural risk} inherent to all post-hoc explanations. Karimi et al. \cite{karimi2022survey} survey algorithmic recourse with some fairness discussion, but treat recourse fairness as a cost-optimization problem rather than as a special case of explanation fairness. Pessach and Shmueli \cite{pessach2022review} survey ML fairness broadly; Dai et al. \cite{dai2024comprehensive} survey trustworthy GNNs; Chen et al. \cite{chen2024fairness} survey fairness-aware GNNs, all without recognizing that the fairness of GNN \emph{explanations} is a dimension independent of prediction fairness.

The shared assumption across all prior surveys is that explanation fairness is reducible either to outcome fairness (fairness surveys) or to explanation quality (XAI surveys). This survey argues that explanation fairness should be treated as a distinct field, and this is the conceptual contribution from which all others follow.

\subsection{Survey Organization}

Section~\ref{sec:background} introduces background on fairness definitions and XAI taxonomies. Section \ref{sec:framework} includes the unifying conditional invariance framework with five design axioms, the identifiability proposition, and the post-hoc vs. in-processing vs.\ intrinsic comparison. Section \ref{sec:mechanisms} identifies the three generative mechanisms by which explanation inequity arises and traces the full causal chain from data imbalance to recourse disparity. Section \ref{sec:attribution_fairness} surveys the fairness of feature attribution methods. Section \ref{sec:procedural} covers procedural fairness metrics, including the in-processing GCIG/FairX framework (Section \ref{sec:gcig}) and the stability-based GESD/FEU framework (Section \ref{sec:gesd}). Section \ref{sec:multiobj} treats multi-objective optimization. Section \ref{sec:cf_recourse} surveys counterfactual and recourse fairness. Section \ref{sec:graph} addresses graph explanation fairness. Section \ref{sec:intersectional} treats intersectional fairness and gerrymandering, including the MESD intersectional procedural metric (Section \ref{sec:mesd}). Section \ref{sec:beyond} surveys rule, example, concept, and neurosymbolic methods, including epistemic fairness. Section \ref{sec:auditing} presents auditing frameworks, the standard evaluation workflow, and an explicit statement of what explanation fairness cannot guarantee. Section \ref{sec:domains} covers domain applications across four high-stakes domains with cross-domain synthesis. Section \ref{sec:open} presents the research agenda: five imperatives and seven prioritized open problems. Section \ref{sec:conclusion} concludes.

\section{Background} \label{sec:background}
\subsection{Fairness Definitions in Machine Learning}
Let $\mathcal{X}$ denote an input feature space and $\mathcal{Y}$ a label space. Let $A$ denote a protected attribute (e.g., race, gender, or age) taking values in a finite set $\mathcal{A}$; when multiple protected attributes are present, $A$ may be replaced by a vector $(A_1, \ldots, A_m)$ with value space $\mathcal{A} = \mathcal{A}_1 \times \cdots \times \mathcal{A}_m$. A classifier $f : \mathcal{X} \to \mathcal{Y}$ is trained on data $\mathcal{D} = \{(x_i, a_i, y_i)\}_{i=1}^n$. Fairness literature has produced an extensive definition that can be broadly classified as group fairness, individual fairness, and causal fairness \cite{corbett2023measure,mitchell2021algorithmic,verma2018fairness}.

\subsubsection{Group Fairness}
Group fairness criteria require that statistical properties of model predictions be approximately equal across protected groups. \emph{Statistical parity} (demographic parity) \cite{dwork2012fairness} requires:
\begin{equation}
  P(\hat{Y}=1 \mid A=a) = P(\hat{Y}=1 \mid A=b),
  \label{eq:sp}
\end{equation} 
\emph{Equalized odds} \cite{hardt2016equality} additionally conditions on true labels: $P(\hat{Y}=1 \mid A=a, Y=y) = P(\hat{Y}=1 \mid A=b, Y=y)$ for all $a, b, y$. \emph{Equal opportunity} \cite{hardt2016equality} restricts this to the positive class. \emph{Predictive parity} requires equal positive predictive values: $P(Y=1 \mid \hat{Y}=1, A=a) = P(Y=1 \mid \hat{Y}=1, A=b)$ \cite{chouldechova2017fair}. The impossibility theorem results show that statistical parity, equalized odds, and calibration cannot all hold simultaneously when base rates differ across groups \cite{kleinberg2016inherent,chouldechova2017fair, bell2023possibility}, a constraint directly relevant to multi-objective optimization (Section \ref{sec:multiobj}).

\subsubsection{Individual Fairness}

Individual fairness requires that similar individuals receive similar outcomes: $d_\mathcal{Y}(f(x),\\ f(x')) \leq L \cdot d_\mathcal{X}(x, x')$ for all $x, x' \in \mathcal{X}$, where $d_\mathcal{X}$ is a task-specific similarity metric and $L$ is a Lipschitz constant \cite{dwork2012fairness, mukherjee2020learning, petersen2021post}. The metric $d_\mathcal{X}$ encodes domain knowledge about which differences between individuals are relevant \cite{ilvento2019metric,mukherjee2020learning}. \emph{Average individual fairness} \cite{sharif2019average, li2023accurate} aggregates individual fairness violations across a population, providing a more tractable audit criterion. \emph{Counterfactual fairness} \cite{kusner2017counterfactual, wu2019counterfactual, ma2023learning}, a special case tightly connected to causal modeling, requires that a prediction be unchanged when the protected attribute of an individual is changed in a causally consistent way. Fleisher \cite{fleisher2021s} argues that individual and counterfactual fairness are insufficient alone for group equity, motivating multi-metric approaches.

\subsubsection{Causal Fairness}

Causal fairness approaches model the data-generating procedure using a structural causal model (SCM) $\mathcal{M} = (\mathcal{G}, F, P_U)$ consisting of a directed acyclic graph, a set of structural equations, and a distribution over exogenous variables \cite{pearl2009causality, zhang2018fairness}. \emph{Path-specific counterfactual fairness} \cite{chiappa2019path, hatano2025path} allows discrimination through admissible causal paths while blocking others. \emph{No unresolved discrimination}~\cite{kilbertus2017avoiding} requires that the effect of $A$ on $\hat{Y}$ not flow through any path that is not mediated by a resolving variable. These frameworks are relevant to explanation fairness because causal graph structure directly informs which explanations are faithful to the underlying data-generating process~\cite{karimi2022survey}.

%\subsubsection{Procedural Fairness} Procedural fairness originates in social psychology~\cite{thibaut1975procedural} and political philosophy~\cite{rawls1971theory} as the principle that fair outcomes must arise through fair processes. In ML, it requires that the decision criteria applied to similar individuals (or groups) not vary systematically across protected groups~\cite{grgic2018beyond, dierckx2022minorities}. Wang et al.~\cite{wang2024procedural} define group procedural fairness formally: ``similar data points in two groups should have similar decision process or logic.'' Germino et al.~\cite{germino2025ediff} operationalize this using SHAP comparisons (Section~\ref{sec:procedural}). Rueda et al.~\cite{rueda2024procedural} argue from a medical ethics perspective that procedural fairness demands explainability in AI-based resource allocation as a matter of justice, extending the concept to normative claims about the right to explanation.

\subsection{Explainable Artificial Intelligence: A Taxonomy}
XAI methods can be taxonomized along several dimensions: \emph{scope} (local vs.\ global), \emph{model dependence} (model-agnostic vs.\ model-specific), \emph{explanation type} (feature importances, counterfactuals, rules, examples, concepts), and \emph{timing} (in-process vs.\ post hoc) \cite{ schwalbe2024comprehensive, speith2022review, linardatos2020explainable}.

\subsubsection{Feature Attribution Methods}
SHAP (SHapley Additive exPlanations)~\cite{lundberg2017unified} derives importance scores from cooperative game theory. The SHAP value for feature $j$ is:
\begin{equation}
  \phi_j(f, x) = \sum_{S \subseteq \mathcal{F} \setminus \{j\}} \frac{|S|!(|\mathcal{F}| - |S| - 1)!}{|\mathcal{F}|!} \left[ v_{S \cup \{j\}}(x) - v_S(x) \right],
  \label{eq:shap}
\end{equation}
where $v_S(x) = \mathbb{E}[f(X) \mid X_S = x_S]$ is the model output conditioned on features $S$. SHAP values satisfy the Shapley axioms: efficiency, symmetry, dummy, and additivity \cite{shapley1953value}. FastSHAP~\cite{jethani2021fastshap} approximates these values through an amortised neural explainer $\hat{\phi}_\theta : \mathcal{X} \to \mathbb{R}^d$ trained to minimize the expected surrogate loss $\mathbb{E}_{x, S}\!\left[(v_{S}(x) - v_{\emptyset}(x) - \mathbf{1}_S^\top \hat{\phi}_\theta(x))^2\right]$, where $\mathbf{1}_S$ is the binary mask selecting the features in subset $S$ and $v_S(x)$ is the conditional expectation as above. This enables real-time explanation generation critical for EDiff's training-time computation~\cite{germino2025ediff}. SHEAR \cite{ wang2024feature} further accelerates SHAP by selecting a contributive cooperator subset of features.

LIME (Local Interpretable Model-agnostic Explanations) \cite{ribeiro2016should} builds a local linear surrogate around the input by minimizing:
\begin{equation}
  \arg\min_{g \in \mathcal{G}} \mathcal{L}(f, g, \pi_x) + \Omega(g),
  \label{eq:lime}
\end{equation}
where $\pi_x(z) = \exp(-d(x,z)^2/\sigma^2)$ is a locality kernel and $\Omega$ is a complexity penalty. Gradient-based methods, including Saliency Maps \cite{simonyan2014deep}, GradCAM \cite{selvaraju2017grad}, and Integrated Gradients \cite{sundararajan2017axiomatic}, leverage the model's gradient structure to compute feature attributions, often visualized as saliency maps in image tasks. Permutation Feature Importance \cite{breiman2001random} provides a model-agnostic alternative by measuring prediction degradation when individual feature columns are shuffled.

\subsubsection{Counterfactual Explanation Methods}
Counterfactual explanations answer: ``what is the minimal change to input $x$ to flip the prediction?'' \cite{stepin2021survey, slack2021counterfactual}. Wachter et al. \cite{wachter2017counterfactual} formulate this as: $\arg\min_{x'} \lambda (f(x') - y')^2 + d(x, x')$, where $d$ measures feature change cost. DiCE \cite{mothilal2020explaining} extends this to generate $k$ diverse counterfactuals. FACE \cite{poyiadzi2020face} constrains counterfactuals to feasible paths in the data manifold. Growing Spheres \cite{laugel2018inverse} uses a sphere-growing procedure to find the closest counterfactual in feature space. Causal recourse methods \cite{karimi2021algorithmic} incorporate SCMs to ensure counterfactuals respect causal structure.

\subsubsection{Rule-Based and Example-Based Methods}
ANCHOR \cite{ribeiro2018anchors} produces if-then rules that are sufficient conditions for a prediction with high precision: $P(f(z) = f(x) \mid A(z)) \geq \tau$, where $A(z)$ denotes the rule firing on instance $z$. Prototype and criticism methods \cite{kim2016examples, obermair2023example} identify representative and unrepresentative training examples. ProtoDash \cite{gurumoorthy2019efficient} provides a fast selection algorithm for prototypes. Influence functions \cite{koh2017understanding, koh2019accuracy, du2014influence} identify which training examples most influence a given prediction, enabling instance-based explanations at the training level.

\subsubsection{Concept-Based and Neurosymbolic Methods}
TCAV (Testing with Concept Activation Vectors)  \cite{kim2018interpretability, pfau2021robust} measures the sensitivity of a prediction to user-defined high-level concepts by training binary classifiers on concept datasets and projecting their decision boundaries into activation space. Concept Bottleneck Models \cite{koh2020concept, chauhan2023interactive} explicitly predict human-interpretable concepts as intermediate representations. Neurosymbolic methods \cite{wagner2021logical, morales2024univariate} integrate logical constraint satisfaction with neural learning, offering the possibility of encoding fairness axioms as first-order logical clauses. These methods are surveyed in Section \ref{sec:beyond}.

\section{A Unifying Framework: Explanation Fairness as Conditional Invariance}\label{sec:framework}

The diverse metrics surveyed in this paper,$\Delta\text{REF}$, $\Delta\text{GESD}$, GPF$_\text{FAE}$, EDiff, recourse cost disparity, and subgroup calibration, are not independent inventions but approximate operationalizations of a single underlying invariance condition. Making this structure explicit transforms the survey from a literature synthesis into a theoretical contribution by revealing what all these methods are \emph{trying} to achieve, why they succeed or fail in different contexts, and what a complete theory of explanation fairness would require.

\subsection{The Core Invariance Principle}
\begin{definition}[Explanation Fairness as Conditional Invariance] \label{def:ef}
Let $f: \mathcal{X} \to \mathcal{Y}$ be a model and $E: \mathcal{X} \to \mathcal{E}$ an explanation function mapping inputs to an explanation space. Let $A$ denote a protected attribute taking values in a finite set $\mathcal{A}$ (e.g., $\mathcal{A} = \{0,1\}$ for binary attributes). Let $X_\text{rel} \in \mathcal{X}_\text{rel} \subseteq \mathcal{X} \setminus \{A\}$ denote the vector of \emph{task-relevant} features (i.e., all features except the protected attribute and its proxies). An explanation function $E$ is \emph{fair with respect to $f$ and $A$} if and only if:
\begin{equation}
  P(E(X) \in \cdot \mid X_\text{rel} = x_\text{rel},\, A = a) = P(E(X) \in \cdot \mid X_\text{rel} = x_\text{rel},\, A = b)
  \quad \forall\, a, b \in \mathcal{A},\; \forall\, x_\text{rel} \in \mathcal{X}_\text{rel},
  \label{eq:ef_invariance}
\end{equation}
where the left- and right-hand sides denote the conditional distribution of the explanation $E(X)$ over the population of individuals sharing the same task-relevant features $x_\text{rel}$ but differing in their protected attribute value.
\end{definition}

Equation~\ref{eq:ef_invariance} says: \emph{among individuals who are identical in all task-relevant features, the distribution of explanations should not depend on which protected group they belong to}. The distributional equality is over individuals in the population who share the same $x_\text{rel}$, not over a single deterministic evaluation---this is what makes the condition non-trivial even for deterministic explanation functions, because the conditioning partitions the population into cells indexed by $x_\text{rel}$, and within each cell the explanation distribution (over different individuals) must be invariant to $A$. A protected attribute should neither appear directly in explanations nor route through proxy features to alter the explanation distribution. This is the explanation-space analog of counterfactual fairness~\cite{kusner2017counterfactual}, extended from prediction space to explanation space.

\subsection{Existing Metrics as Partial operationalizations}
Each metric in the literature approximates Equation \ref{eq:ef_invariance} under different assumptions and with different tradeoffs, as summarized in Table \ref{tab:unification}. $\Delta\text{VEF}$ \cite{zhao2023fairness} measures the violation in expectation over explanation \emph{quality} (a scalar projection of $E$), making it computationally tractable but insensitive to shape differences in the explanation distribution. GPF$_\text{FAE}$ \cite{wang2024procedural} measures the Maximum Mean Discrepancy between matched explanation distributions, a stronger test of Eq.~\ref{eq:ef_invariance} that captures distributional differences beyond the mean, but requires a matching procedure that may itself be confounded. EDiff~\cite{germino2025ediff} operationalizes the invariance most directly by using counterfactual pairs, which, under the assumption that the counterfactual is generated in a causally consistent manner, eliminates confounding from legitimate group differences. However, if causal downstream effects of $A$ on other features are not propagated, the counterfactual may not represent a true causal intervention. The authors of EDiff mitigate this through a learned surrogate (FastSHAP) but do not fully resolve it. Recourse cost disparity metrics \cite{gupta2019equalizing,ustun2019actionable} operationalize invariance in the action space rather than the explanation space.

Three further metrics extend the framework in complementary directions. GCIG~\cite{popoola2026gcig} operationalizes Eq.~\ref{eq:ef_invariance} from the \emph{training side}: rather than measuring the invariance post-hoc, it enforces it during learning by computing Integrated Gradients relative to group-conditional baselines $b_{y,g} = \mathbb{E}[X \mid Y=y, A=g]$ and penalizing cross-group attribution variance as a differentiable regularization loss. In the linear case, this reduces exactly to $\|w \odot (b_{y,1} - b_{y,0})\|_2^2$, the  elementwise-weighted squared norm of group mean differences, providing a closed-form characterization of what Eq.\ref{eq:ef_invariance} demands of the model parameters. GESD \cite{popoola2025gesd} operationalizes the invariance along a \emph{stability} dimension: rather than comparing explanation vectors across groups, it measures how much each group's explanation changes under input perturbations, capturing a robustness-oriented violation of the invariance that mean-comparison metrics miss entirely. MESD \cite{popoola2025mesd} extends GESD to \emph{intersectional subgroups} via Cartesian product group construction, empirical-Bayes shrinkage for sparse cells, and CVaR tail weighting, ensuring the invariance condition is tested not only at the marginal group level but across all cells of the intersectional subgroup space, directly preventing the fairness gerrymandering that marginal tests permit.

\begin{table}[t]
\caption{Existing explanation fairness metrics as partial operationalizations of the conditional invariance principle (Eq.~\ref{eq:ef_invariance}).}
\label{tab:unification}
\small
\begin{tabular}{p{1.8cm} p{2.2cm} p{2.5cm} p{2.0cm} p{1.6cm} p{1.3cm}}
\toprule
\textbf{Metric} & \textbf{Projection of $E$} & \textbf{Comparison mode} & \textbf{Confound control} & \textbf{Enforcement} & \textbf{Complete-ness} \\
\midrule
$\Delta\text{VEF}$ & Scalar quality $q(x)$ & Group means & None & Post-hoc & Partial \\
$\Delta\text{REF}$ & Binary quality & Group ratios & None & Post-hoc & Partial \\
GPF$_\text{FAE}$ & Full vector $e \in \mathbb{R}^d$ & MMD on matched pairs & Similarity matching & Post-hoc & Moderate \\
EDiff & Full vector $\phi(x)$ & Counterfactual pairs & $A$ flipped, $x$ fixed & In-processing & Strong \\
Recourse CD & Action cost $c(a,x)$ & Group means & Actionability constraint & Post-hoc & Domain-specific \\
Eq.\ explainability & Wasserstein on $\phi$ & Distribution distance & None & Post-hoc & Moderate \\
GCIG & IG vector (normalised) & Group-conditional baselines & Label $y$ fixed & In-processing & Strong \\
GESD & Stability score $S(x)$ & Perturbation variance & Perturbation neighbourhood & Post-hoc & Moderate \\
MESD & Stability (intersectional) & CVaR-weighted tail disparity & Label + EB shrinkage & Post-hoc & Strong \\
\bottomrule
\end{tabular}
\end{table}

\subsection{Axioms of a Fair Explanation System}
\label{sec:axioms}

The conditional invariance principle (Definition \ref{def:ef}) defines explanation fairness as the \emph{absence} of group-differential explanation behavior. A constructive theory of explanation fairness, specifying what a fair explanation system positively \emph{must} do, requires extracting design axioms from the invariance condition and the metrics literature. We propose five such axioms. Together, they constitute the minimal requirements that any explanation system in a high-stakes domain should satisfy.

\begin{description}
  \item[\textbf{Axiom 1 (Process Consistency).}] The explanation function $E$ must apply the same decision criteria to similar individuals regardless of their protected group membership: if $d_\mathcal{X}(x, x') \leq \delta$ and $A(x) \neq A(x')$, then $\|E(x) - E(x')\| \leq \epsilon(\delta)$ for some continuity modulus $\epsilon$. This is the explanation-space is comparable to individual fairness~\cite{dwork2012fairness}, and is directly operationalized by GPF$_\text{FAE}$ (Eq. \ref{eq:gpf}).

  \item[\textbf{Axiom 2 (Counterfactual Stability).}] The explanation must be stable under protected-attribute counterfactuals: for any individual $x$ and attribute values $a, b \in \mathcal{A}$, $E(x \mid A=a) \approx E(x \mid A=b)$, where $x$ is otherwise identical. This is the formal statement of Definition~\ref{def:ef} and is directly operationalized by EDiff (Eq. \ref{eq:ediff}).

  \item[\textbf{Axiom 3 (Distributional Parity).}] The population-level distribution of explanations should not vary systematically across protected groups: $\mathcal{W}_p\bigl(\{E(x): x \in \mathcal{D}_a\}, \{E(x): x \in \mathcal{D}_b\}\bigr) \leq \epsilon$ for all $a, b \in \mathcal{A}$. This is operationalized by $\Delta\text{VEF}$ and equalized explainability \cite{wang2024achieving}.

  \item[\textbf{Axiom 4 (Actionability Symmetry).}] The cost required to act on an explanation, to change one's situation from denied to approved, must not differ systematically across protected groups: $\mathbb{E}[c(x) \mid A=a] \approx \mathbb{E}[c(x) \mid A=b]$, where $c(x)$ is the recourse cost for instance $x$. This is operationalized by recourse cost disparity metrics \cite{ustun2019actionable,gupta2019equalizing}.

  \item[\textbf{Axiom 5 (Epistemic Accessibility).}] An explanation must be interpretable and actionable, not only in information-theoretic terms but for the actual recipients. This means the ability to comprehend, contest, and act on an explanation must be equitable across demographic groups. This axiom is not captured by any current metric and requires human-centered evaluation (Section \ref{sec:epistemic}).
\end{description}

\begin{remark}
Axioms 1--4 are \emph{measurable} by existing metrics, though all rely on approximating assumptions detailed in Table~\ref{tab:unification}. Axiom 5 is currently \emph{unmeasured} by any quantitative metric in the literature. This gap defines the epistemic fairness research agenda. Notably, GCIG directly enforces Axiom 3 at training time, while GESD and MESD operationalize a robustness-oriented variant of Axiom 2 (stability under perturbation rather than under attribute counterfactuals), and MESD additionally operationalizes Axiom 3 at the intersectional level.
\end{remark}

The five axioms are not mutually redundant. A system satisfying Axiom 1 may violate Axiom 2 (group-level distributions may be matched while individual counterfactuals diverge). A system satisfying Axioms 1--4 may violate Axiom 5 (attribution parity does not imply comprehension parity). And no current post-hoc system can provably satisfy Axiom 2 for all inputs, as the following proposition makes precise.

\subsection{The Identifiability Barrier}
\label{sec:identifiability}

\begin{proposition}[Identifiability of Explanation Fairness]
\label{prop:identifiability}
Let $E: \mathcal{X} \to \mathcal{E}$ be any post-hoc explanation method that accesses the model $f$ only through its input-output behavior. Then there exist models $f_\text{fair}$ and $f_\text{unfair}$ such that:
\begin{enumerate}[(a)]
  \item $f_\text{fair}$ satisfies Definition~\ref{def:ef} (is explanation-fair), and
  \item $f_\text{unfair}$ violates Definition~\ref{def:ef} (is explanation-unfair), yet
  \item $E$ assigns identical explanation distributions to both: $\{E(x; f_\text{fair})\}_{x \in \mathcal{X}} \overset{d}{=} \{E(x; f_\text{unfair})\}_{x \in \mathcal{X}}$.
\end{enumerate}
\end{proposition}

\begin{proof}[Proof sketch]
The construction adapts the fairwashing approach of Anders et al.~\cite{anders2020fairwashing} and A\"ivodji et al.~\cite{aivodji2019fairwashing}. Let $f_\text{unfair}$ be a classifier that violates Definition~\ref{def:ef}. Construct a proxy model $g$ that agrees with $f_\text{unfair}$ on the in-distribution data but is designed so that on the auditor's query distribution (the points where $E$ evaluates the model), $g$ behaves identically to a fair model $f_\text{fair}$. Because $E$ accesses $g$ only through input-output queries, $E$ cannot distinguish $g$ from $f_\text{fair}$, because they produce identical outputs on all queried points. Therefore $E(x; g) = E(x; f_\text{fair})$ for all $x$ in the query set. Yet $g$ agrees with $f_\text{unfair}$ on the deployment distribution, so the deployed model remains unfair, even though the explanation method certifies fairness. The existence of such $g$ is guaranteed by the model-agnostic access assumption and is constructively demonstrated in the fairwashing literature~\cite{aivodji2019fairwashing,shahin2022washing}.
\end{proof}

Proposition~\ref{prop:identifiability} has an immediate corollary that, to our knowledge, has not previously been stated as a standalone claim in the literature:

\begin{remark}[The Hard Limit of Post-Hoc Explanation]
\label{rem:hardlimit}
\textbf{Post-hoc explanation methods are structurally incapable of certifying procedural fairness.} This is not an engineering limitation to be solved by better algorithms; it is a consequence of model-agnostic access. A system that passes all post-hoc explanation fairness audits may still be profoundly explanation-unfair. Auditors who rely exclusively on post-hoc inspection receive a lower bound on procedural fairness, never a certificate \cite{slack2020fooling}.
\end{remark}

Table~\ref{tab:posthoc_vs_intrinsic} makes the architectural implications concrete, comparing post-hoc, in-processing, and intrinsically interpretable approaches on their ability to satisfy each of the five axioms.

\begin{table}[t]
\caption{Ability of three explanation architectures to satisfy the five axioms of explanation fairness. ``Partial'' means the axiom can be approximately satisfied with additional constraints; ``No'' means structural incapability; ``Yes'' means the architecture admits satisfaction by design.}
\label{tab:posthoc_vs_intrinsic}
\small
\begin{tabular}{p{2.8cm} p{1.8cm} p{1.8cm} p{1.8cm} p{1.8cm} p{1.8cm}}
\toprule
\textbf{Architecture} & \textbf{A1: Process Consistency} & \textbf{A2: CF Stability} & \textbf{A3: Distr.\ Parity} & \textbf{A4: Actionability} & \textbf{A5: Epistemic} \\
\midrule
Post-hoc (SHAP, LIME) & Partial & \textbf{No} & Partial & Partial & No \\
In-processing (CFA, EDiff) & Yes & Partial & Yes & Yes & No \\
Intrinsic (CBMs, scorecards) & Yes & Yes & Yes & Yes & Partial \\
\bottomrule
\end{tabular}
\end{table}

\subsubsection{The Identifiability Barrier (Causal View)}

The condition compares $E(x \mid A=a)$ and $E(x \mid A=b)$ for the \emph{same} $x$, but in real data, we observe individuals with only one value of $A$. Detecting violations, therefore, requires either (a) counterfactual individuals (as in EDiff), (b) matching assumptions that may be confounded, or (c) access to the causal graph of the data-generating process. This is an instance of the fundamental \emph{identifiability problem} in causal inference~\cite{pearl2009causality}: explanation unfairness is an interventional quantity, but we observe only the observational distribution.

The vulnerability demonstrations of Anders et al.~\cite{anders2020fairwashing} and Shamsabadi et al.~\cite{shahin2022washing} can be reinterpreted through this lens: they show that the identifiability barrier is not merely statistical but \emph{computational}, even with infinite data, no output-based post-hoc method can certify that Eq.~\ref{eq:ef_invariance} holds, because the required counterfactual information is not encoded in model outputs. Proposition~\ref{prop:identifiability} generalizes these specific attack constructions into a structural claim about the entire class of model-agnostic post-hoc methods, reframing fairwashing not as an adversarial pathology but as a structural consequence of the identifiability gap in post-hoc explanation.

\subsubsection{Causal Roots of Explanation Unfairness}

The identifiability barrier points to a deeper causal structure underlying explanation unfairness. Post-hoc explanation methods are surrogates that approximate the model's true reasoning without access to its internal causal structure. Explanation unfairness arises from three causal mechanisms, each corresponding to a different locus in the data-generating process:

\begin{enumerate}
  \item \textbf{Proxy leakage.} The protected attribute $A$ has causal descendants $Z \subset \mathcal{X}$ (proxy features such as ZIP code for race, or occupation for gender). Even when $A$ is excluded from the model, SHAP distributes credit between $A$ and $Z$ according to their marginal contributions, causing the explanation vector to implicitly encode $A$ through its proxies.
  \item \textbf{Surrogate mismatch.} Post-hoc methods like LIME and SHAP construct local linear approximations whose accuracy depends on the local geometry of the feature manifold. Minority groups whose features occupy sparse or high-curvature regions of the manifold receive systematically less faithful approximations~\cite{begley2020explainability}, creating explanation unfairness even when the model itself is unbiased.
  \item \textbf{Distributional asymmetry.} When the causal structure of the data-generating process differs between groups, the same feature attribution scores carry different causal meanings across groups, producing explanations that are numerically similar but causally inequivalent.
\end{enumerate}

These three mechanisms have different implications for mitigation. Proxy leakage requires causal graph knowledge or proxy auditing. Surrogate mismatch requires distribution-aware explanation methods or ante-hoc interpretable models. Distributional asymmetry may be irremediable without structural interventions in the data-generating process itself. Figure~\ref{fig:pipeline} situates these relationships in a conceptual pipeline from data to human decision, showing where outcome fairness and explanation fairness each intervene

\begin{figure}[t]
\centering
\begin{tikzpicture}[
  node distance = 1.2cm and 1.8cm,
  box/.style  = {rectangle, rounded corners=4pt, draw=black!70, fill=white,
                 minimum width=2.4cm, minimum height=0.85cm, align=center,
                 font=\small},
  fairbox/.style = {rectangle, rounded corners=4pt, draw=black!80,
                    fill=gray!15, minimum width=2.6cm, minimum height=0.85cm,
                    align=center, font=\small\itshape},
  arr/.style  = {-{Latex[length=2mm]}, thick, black!70},
  darr/.style = {-{Latex[length=2mm]}, thick, black!50, dashed},
]

%% Main pipeline (horizontal)
\node[box] (data)  {Training\\Data};
\node[box, right=of data]  (model) {Model\\Representation};
\node[box, right=of model] (expl)  {Explanation\\Mapping};
\node[box, right=of expl]  (human) {Human\\Interpretation};

%% Arrows along main pipeline
\draw[arr] (data)  -- (model);
\draw[arr] (model) -- (expl);
\draw[arr] (expl)  -- (human);

%% Bias injection points
\node[above=0.45cm of data,  font=\scriptsize, text=black!60] {\textit{historical bias}};
\node[above=0.45cm of model, font=\scriptsize, text=black!60] {\textit{proxy leakage}};
\node[above=0.45cm of expl,  font=\scriptsize, text=black!60] {\textit{surrogate mismatch}};
\node[above=0.45cm of human, font=\scriptsize, text=black!60] {\textit{epistemic gap}};

%% Outcome and Explanation Fairness nodes (below)
\node[fairbox, below=1.1cm of model] (of) {Outcome\\Fairness};
\node[fairbox, below=1.1cm of expl]  (ef) {Explanation\\Fairness};

%% Dashed links down
\draw[darr] (model.south) -- (of.north)
            node[midway, left, font=\scriptsize, text=black!55] {$\hat{Y}$};
\draw[darr] (expl.south)  -- (ef.north)
            node[midway, right, font=\scriptsize, text=black!55] {$E(x)$};

%% Arrow between fairness nodes (they influence each other)
\draw[darr, <->] (of.east) -- (ef.west)
                 node[midway, below, font=\scriptsize, text=black!55] {partially independent};

\end{tikzpicture}
\caption{Conceptual pipeline of explanation fairness. Bias enters at multiple stages: historical bias in training data, proxy leakage during model learning, surrogate mismatch during post-hoc explanation, and epistemic gap during human interpretation. Outcome fairness and explanation fairness are both downstream of the model representation, but are \emph{partially independent}: a model can satisfy one while violating the other.}
\label{fig:pipeline}
\end{figure}
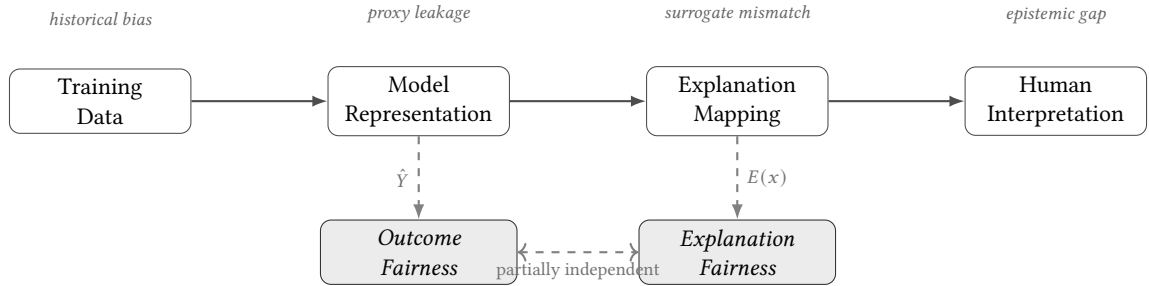

\subsubsection{Implications for the Taxonomy}

The conditional invariance framework also clarifies the \emph{relationships} between the seven survey dimensions: Dimensions 1--2 are direct operationalizations of Eq.~\ref{eq:ef_invariance} in feature-attribution space; Dimension 3 provides the training machinery to enforce it as a constraint; Dimension 4 operationalizes it in action space; Dimension 5 extends it to relational data; Dimension 6 requires it to hold across all intersectional subgroups; and Dimension 7 asks whether it holds when $E$ is a rule, concept, or example. This unifying view suggests that methods developed for one dimension can be adapted to others; for example, EDiff's counterfactual pairing strategy could be extended to GNN subgraph explanations, providing a generative principle for future method development.

\section{Mechanisms of Explanation Inequity}
\label{sec:mechanisms}

The sections that follow define \emph{what} explanation unfairness looks like across different method families. Before proceeding, this section addresses a prior question: \emph{how does explanation inequity arise as a system phenomenon?} Understanding generative mechanisms matters because the appropriate remediation depends entirely on the causal pathway being addressed. We identify three generative pathways through which explanation unfairness is produced, corresponding to different loci in the pipeline from data to human decision (Figure \ref{fig:pipeline}).

\subsection{Pathway 1: Representation-Driven Inequity}
The first pathway originates in the model's learned representation. When training data encodes historical discrimination through exclusion, underrepresentation, or proxy correlations, the model learns systematically different representations for different demographic groups. These representational disparities then propagate into explanations, such as SHAP values, gradient attributions, and local surrogates, all of which reflect the feature geometry of the learned representation. If that geometry is asymmetric across groups, the explanations will be asymmetric too, even if the model's predictions appear fair by distributional criteria.

The causal chain is explicit and traceable. \emph{Data imbalance} (minority groups underrepresented in training) causes \emph{latent embedding skew} (minority features occupy sparse, poorly-sampled regions of the representation manifold). This produces \emph{decision boundary asymmetry} (the learned classifier boundary makes different implicit assumptions about feature reliability for different groups). Boundary asymmetry directly causes \emph{attribution distortion} (SHAP and LIME assign systematically different importance to identical feature changes depending on which group the instance belongs to). Attribution distortion then produces \emph{recourse disparity} (counterfactual explanations that are mathematically valid are practically infeasible for minority groups because the attributed decision logic points them toward features that are harder to change). Each link in this chain has empirical support: Begley et al.~\cite{begley2020explainability} for embedding skew to attribution distortion; Ustun et al.~\cite{ustun2019actionable} for attribution distortion to recourse disparity; Ferrario and Loi~\cite{ferrario2022explainability} for temporal degradation of the full chain.

Three specific mechanisms drive this chain. \emph{Homophily leakage} in graph-structured data causes neighborhood aggregation to propagate protected attribute information into node representations even when the attribute is explicitly excluded~\cite{dai2021say}. \emph{Proxy feature dominance} occurs when features causally downstream of the protected attribute (ZIP code, occupation, credit history) absorb attributional credit that would otherwise be assigned to the protected attribute under a causal model, obscuring the true source of discrimination~\cite{molnar2020pitfalls}. \emph{Embedding density asymmetry} places minority groups in sparser regions of the representation manifold, producing higher-variance and lower-fidelity attribution scores~\cite{begley2020explainability}, which we formalized as surrogate mismatch in Section~\ref{sec:mechanisms}.

\emph{Implications}: Representation-driven inequity cannot be fixed by changing the explanation method alone; it requires intervention at the model or data level. Debiased embeddings~\cite{wang2024achieving}, causally constrained training~\cite{karimi2021algorithmic}, and ante-hoc interpretable models~\cite{rudin2019stop} are the principal remediation strategies. Crucially, the full causal chain means that debiasing only one link leaves others intact. For example, equalizing SHAP distributions (fixing attribution distortion) without also equalizing embedding density leaves the upstream cause of disparity unchanged.

\subsection{Pathway 2: Explanation-Model Mismatch}

The second pathway arises from the epistemic limitations of post-hoc explanation methods themselves. Post-hoc explainers are surrogates: they approximate the model's behavior using only query access to its inputs and outputs, without access to its internal causal structure. This creates an irremovable information gap between the explanation and the model's actual reasoning.

The gap manifests in three forms. \emph{Local linearity failure}: LIME's local linear approximation is accurate only where the model's decision boundary is approximately linear. For minority groups whose feature vectors lie near or on complex non-linear boundaries (e.g., the intersection of multiple risk factors in credit scoring), the linear surrogate is systematically less faithful~\cite{molnar2020pitfalls, knab2025lime}. \emph{Marginalization over correlated features}: SHAP's expectation over feature coalitions is technically correct but causally misleading when features are correlated. Attribution credit is distributed between a protected attribute and its proxies according to their correlational structure, not their causal contribution, producing explanations that hide discrimination under a veil of proxy variables~\cite{fuster2022predictably}. \emph{Adversarial exploitation}: because post-hoc methods have no ground truth for comparison, a biased model can be paired with a fairwashing proxy that presents a different model to the explainer, the explanation is faithful to the proxy, not to the deployed model~\cite{slack2020fooling, aivodji2019fairwashing}. Anders et al.~\cite{anders2020fairwashing} prove this cannot be fixed within the post-hoc paradigm.

\emph{Implications}: Explanation-model mismatch is not an engineering failure; it is a structural consequence of model-agnostic access. The appropriate responses are: using model-intrinsic explanation methods (Concept Bottleneck Models~\cite{koh2020concept}, inherently interpretable architectures~\cite{rudin2019stop}); requiring ante-hoc audits that inspect model weights directly; or accepting that post-hoc auditing provides only a lower bound on fairness, not a certificate.

\subsection{Pathway 3: Actionability-Driven Inequity}

The third pathway operates downstream of both the model and the explainer, in the translation from explanation to action. Even when an explanation accurately reports the features driving a decision, the \emph{ability to act} on that explanation varies systematically across demographic groups. This is the epistemic fairness dimension (Section~\ref{sec:epistemic}), and it is the form of inequity least well captured by purely technical metrics.

Actionability-driven inequity takes three forms. \emph{Structural recourse barriers}: a counterfactual explanation prescribing ``increase your income by \$15,000'' may be technically valid but practically inaccessible to a minority applicant facing structural economic barriers~\cite{ustun2019actionable,gupta2019equalizing}. \emph{Temporal obsolescence}: explanations and recourse prescriptions calibrated at deployment become invalid as the data distribution shifts, and minority groups undergoing rapid socioeconomic change experience this obsolescence faster~\cite{ferrario2022explainability,ezzeddine2025}. \emph{Comprehension asymmetry}: explanations equally complex in information-theoretic terms are not equally comprehensible across groups with different educational backgrounds or domain knowledge~\cite{lunich2024explainable,kaur2020interpreting}. An explanation that a data scientist reads in 30 seconds may require an unrepresented individual several minutes, time, access to counsel, and cognitive bandwidth that are differentially available.

\emph{Implications}: Actionability-driven inequity is not addressable by technical metrics alone. Remediation requires participatory design~\cite{rudin2022interpretable}, narrative-sensitive recourse generation, temporal robustness constraints, and human-centered evaluation that measures actual comprehension and action rates rather than information-theoretic proxies.

\subsection{How the Pathways Interact}

The three pathways are not independent. Representation-driven inequity amplifies explanation-model mismatch: minority groups in sparse manifold regions receive both more biased representations \emph{and} less faithful local approximations. Explanation-model mismatch enables adversarial exploitation of actionability asymmetries: a fairwashed explanation may report a recourse path that is feasible for a majority-group applicant but systematically prescribes higher-cost paths for minority applicants. Temporal obsolescence interacts with all three: distributional shift degrades both representational accuracy and surrogate fidelity, while simultaneously increasing recourse infeasibility for groups experiencing the most change.

This interaction structure has a direct practical implication: auditing that addresses only one pathway may leave the others undetected. A complete explanation, a fairness audit must probe all three.

\section{Fairness of Feature Attribution Methods}
\label{sec:attribution_fairness}

Feature attribution methods are the dominant paradigm for post-hoc explanation in deployed ML systems, but they are not inherently fair. The importance scores they assign can differ in quality, faithfulness, and consistency across demographic groups, and they can be manipulated to conceal discriminatory behavior \cite{zhou2022feature, zhou2021evaluating}. This section surveys the main research directions, such as disparate explanation quality, adversarial manipulation and fairwashing, bias in attribution scores, and mitigation strategies.

\subsection{Disparate Explanation Quality Across Groups}

The quality of an explanation can be measure using the following metrics: \emph{fidelity} (does the explanation accurately reflect the model's behavior?) \cite{yeh2019fidelity, wang2023impact, papenmeier2019model}, \emph{stability} (does the explanation remain consistent across nearby inputs?) \cite{fel2022good, lakkaraju2020robust}, and \emph{complexity} (is the explanation concise enough to be actionable?)~\cite{guidotti2018survey, li2023defining}. A model may produce explanations of systematically lower quality on all these dimensions for certain demographic groups, either as a direct consequence of model bias or as an artifact of the explanation method's assumptions.

Dimanov et al.~\cite{dimanov2020you} were among the first to systematically document disparate explanation quality, demonstrating that SHAP values can suppress the apparent importance of protected attributes while the underlying model still relies on them. Their empirical study on recidivism prediction showed that manipulating the SHAP computation to reduce the attributed weight of race does not remove racial discrimination; it merely hides it. This anticipates the impossibility results discussed below and highlights the insufficiency of explanation-based auditing when explanations themselves are not constrained to be faithful.

Begley et al.~\cite{begley2020explainability} examined SHAP and LIME attributions in healthcare predictions and found that fidelity and stability were significantly lower for minority patients. They attributed this to distributional heterogeneity within protected groups (i.e., minority groups are often underrepresented in training data), which leads to higher variance in local linear approximations and SHAP value estimates. The consequence is double jeopardy, meaning the groups most affected by potentially discriminatory medical predictions receive the least reliable explanations.

Strauch et al \cite{strauch2023saliency} extended this analysis to visual explanations in image classification, showing that GradCAM saliency maps exhibit differential attention patterns across demographic groups, with explanations for males being less localized and less aligned with the true discriminating features. This problem is particularly acute in medical imaging, where the quality of explanations directly affects clinicians' ability to interpret and act on model outputs.

Dodge et al. \cite{dodge2019explaining} and Lunich and Keller~\cite{lunich2024explainable} studied the relationship between explanation simplicity and perceived fairness in a human subject setting, finding that simplified decision tree explanations were perceived as fairer than complex model explanations by students, even when the underlying decisions were equally biased. This highlights that explanation fairness is not purely a technical property but has important social and perceptual dimensions.

\subsection{Adversarial Manipulation of Explanations and Fairwashing}

Among the most significant findings in the explanation fairness literature is the vulnerability of post-hoc methods to deliberate manipulation \cite{senevirathna2024deceiving}. Slack et al.~\cite{slack2020fooling} demonstrated a concrete attack in which a biased classifier is paired with a detection mechanism that identifies when a query comes from an auditor (by detecting distribution shift from training data) and substitutes an innocuous classifier for explanation purposes. SHAP and LIME, when applied in this setting, consistently report apparently fair explanations, whereas the deployed model discriminates. Their experiments covered income prediction and recidivism, showing that both SHAP and LIME could be comprehensively deceived with minimal reduction in predictive accuracy.

Aivodji et al.~\cite{aivodji2019fairwashing} formalized this as \emph{fairwashing} as the construction of a proxy model $g$ that is close to an unfair model $f$ in prediction space but presents fair explanations to an explanation-based auditor. Their theoretical analysis characterizes the set of classifiers that can be fairwashed and shows that the fairwashing problem is tractable whenever the explanation method is model-agnostic and post-hoc. More follow-ups ~\cite{aivodji2021characterizing, shahin2022washing} prove that fairwashing remains possible even when predictions are publicly observable, as long as the auditor relies on post-hoc explanations rather than directly inspecting the model. Lafarague et al.~\cite{lafargue2025exposing} introduce \emph{stealthy sampling}, in which LIME's locality kernel is exploited to present systematically different neighborhoods to auditors and affected individuals, concealing discriminatory patterns within the reported local linear approximation.

Anders et al.~\cite{anders2020fairwashing} demonstrated a fundamental vulnerability: under natural axiomatic conditions on explanation methods, there exist adversarial constructions of datasets and classifiers for which the explanation method cannot distinguish a fair from an unfair model. Their work shows specific attack constructions rather than a universal impossibility theorem; we generalize this to a structural identifiability claim in Proposition \ref{prop:identifiability}. This result does not preclude useful explanation auditing in practice but establishes a ceiling on the theoretical guarantees any post-hoc method can offer.

Baniecki and Biecek~\cite{baniecki2024adversarial} provide a broad survey of adversarial XAI attacks, cataloging attacks on SHAP, LIME, and gradient methods targeting fidelity, stability, and fairness. Cararanza et al.~\cite{carranza2023deceptive} construct attacks that manipulate explanation attention away from protected attributes while preserving outcome fairness metrics, showing that fairwashing detection is NP-hard under hardness assumptions analogous to those in weak agnostic learning.

\subsection{Bias in SHAP and LIME Attributions}

Beyond adversarial manipulation, feature attribution methods may exhibit systematic biases under ordinary deployment conditions. A fundamental issue with SHAP is its handling of feature correlations. The SHAP value computation assumes feature independence when $X_{\bar{s}}$ is drawn from its marginal rather than conditional distribution, but when this expectation is computed by marginalizing over the training distribution, correlated features receive distorted attributions \cite{ luo2020parameterized, takefuji2025limitations, hwang2025shap}. This is directly relevant to fairness, where protected attributes are typically correlated with proxy features (ZIP code and race, income, and gender), causing SHAP to distribute attribution credit between them in ways that are both technically correct under the model's learned correlations and socially misleading about the true source of discrimination~\cite{fuster2022predictably, tiwari2024bias}.

Molnar et al.~\cite{molnar2020pitfalls} systematically survey pitfalls in explainable ML, including the attribution bias introduced by correlated features in SHAP, the extrapolation problem in LIME (where the local linear approximation is evaluated at points far from the training distribution), and the instability of ANCHOR rules under small input perturbations. Each of these pitfalls has differential impacts across demographic groups (i.e., minority groups whose input features lie in sparse regions of the feature space are particularly susceptible to extrapolation artifacts and high-variance LIME attributions).

Ferrario and Loi~\cite{ferrario2022explainability} and Slack et al. ~\cite{slack2021reliable} study the temporal stability of SHAP-based explanations, showing that for groups experiencing distributional shifts (due to economic mobility, migration, or social change), explanation stability is lower and actionable recourse becomes obsolete faster. In credit-scoring models, minority groups experienced significantly higher rates of explanation obsolescence, a phenomenon in which a previously valid, actionable explanation no longer produces the desired outcome, raising temporal equity concerns that are absent from standard fairness analyses.

\subsection{Equalized Explainability and Mitigation}

Wang and Wu~\cite{wang2024achieving} propose \emph{equalized explainability} as an objective, which states that the distribution of explanation vectors across demographic groups should be identical. They achieve this through a data reconstruction approach that modifies the training data to reduce distributional disparities in feature importance scores without retraining the underlying model. Their approach is post-hoc and can be applied on top of any feature attribution method, making it broadly applicable. The equalized explainability condition can be formalized as:
\begin{equation}
  \mathcal{W}_p\left(\{\phi(x) : x \in \mathcal{D}_1\}, \{\phi(x) : x \in \mathcal{D}_0\}\right) \leq \epsilon,
\end{equation}
where $\mathcal{W}_p$ is the $p$-Wasserstein distance and $\phi(x)$ is the feature attribution vector for instance $x$.

Waller et al.~\cite{waller2024identifying} design an argumentation-based approach to identify sources of model bias using explanations, providing interpretable causal chains from features to observed disparities. Their method can diagnose whether a model's bias originates from specific features or from their interactions, informing targeted interventions. Zhao et al.~\cite{zhao2023fairness} propose the Comprehensive Fairness Algorithm (CFA), which jointly optimizes traditional fairness, explanation fairness, and predictive utility. Table \ref{tab:attribution} show key method that raise concern of fairness in feature attribution method

\begin{table}[]
\caption{Key feature attribution fairness methods.}
\label{tab:attribution}
\small
\begin{tabular}{p{2.7cm} p{1.7cm} p{2.5cm} p{2.2cm} p{1.2cm}}
\toprule
\textbf{Method} & \textbf{Attribution} & \textbf{Fairness notion} & \textbf{Strategy} & \textbf{Year} \\
\midrule
Slack et al.~\cite{slack2020fooling} & SHAP/LIME & Fairwashing attack & Adversarial & 2020 \\
Aivodji et al.~\cite{aivodji2019fairwashing, aivodji2021characterizing} & Feature imp. & Fairwashing formalization & Theoretical & 2019 \\
Anders et al.~\cite{anders2020fairwashing} & Feature imp. & Impossibility & Axiomatic & 2020 \\
Dimanov et al.~\cite{dimanov2020you} & SHAP & Procedural bias & Constrained opt. & 2020 \\
Lafarague et al.~\cite{lafargue2025exposing} & LIME & Stealthy sampling & Attack/detect & 2023 \\
Wang \& Wu~\cite{wang2024achieving} & SHAP & Equalized explainability & Post-processing & 2024 \\
Begley et al.~\cite{begley2020explainability} & SHAP/LIME & Quality disparity & Measurement & 2020 \\
Waller et al.~\cite{waller2024identifying} & Argumentation & Bias diagnosis & Diagnostic & 2024 \\
Shamsabadi et al.~\cite{shahin2022washing} & Feature imp. & Detection impossibility & Theoretical & 2022 \\
Ferrario \& Loi~\cite{ferrario2022explainability} & SHAP & Temporal fairness & Measurement & 2022 \\
\bottomrule
\end{tabular}
\end{table}

\subsection{Synthesis: Mechanisms of Attribution Fairness Failure}

Across the literature, four recurring mechanisms explain why feature attribution methods fail on fairness dimensions. \emph{Representation disparity} (Begley et al.~\cite{begley2020explainability}) arises from training data imbalance, where minority groups occupy sparser regions of the feature space, leading to higher-variance local approximations. \emph{Adversarial manipulation} ( Aïvodji et al.~\cite{aivodji2019fairwashing}, Anders et al.~\cite{anders2020fairwashing}) exploits the model-agnostic surface access of post-hoc methods. Because SHAP and LIME query only model outputs, they cannot detect when the model used for querying differs from the deployed one. \emph{Correlation laundering} (Dimanov et al.~\cite{dimanov2020you}, Molnar et al.~\cite{molnar2020pitfalls}) occurs when proxy features absorb protected attribute attribution, hiding discrimination behind apparently neutral feature importance scores. \emph{Temporal degradation} (Ferrario and Loi~\cite{ferrario2022explainability}) refers to the fact that explanation accuracy degrades under distributional shift, and this degradation is systematically faster for groups experiencing greater socioeconomic mobility.

These mechanisms are not independent, where adversarial manipulation often exploits representation disparity, and correlation laundering amplifies temporal degradation. Mitigating attribution unfairness, therefore, requires addressing the \emph{causal} rather than just the statistical relationship between attributions and protected attributes, a theme formalized in the conditional invariance framework (Section~\ref{sec:framework}).

\section{Procedural Fairness and Explanation Difference Metrics}
\label{sec:procedural}

A fundamental limitation of outcome-oriented fairness metrics (e.g., Equalized odds \cite{hardt2016equality}, Statistical parity \cite{chouldechova2017fair}) is their exclusive focus on the equity of model \emph{outputs} without examining the equity of the \emph{process} that produces them. Two models with identical prediction distributions may evaluate different demographic groups on entirely different criteria, creating structurally unequal decision processes that are invisible to outcome-based auditing \cite{balagopalan2022road, jain2020biased, dai2021will}. This section surveys the emerging body of work on procedural fairness in machine learning, with emphasis on formal metrics grounded in feature attribution explanations.

\subsection{Philosophical Foundations of Procedural Fairness}

The concept of procedural fairness has deep roots in the social sciences and law. Thibaut and Walker~\cite{thibaut1975procedural}, Thibaut et al. ~\cite{thibaut1973procedural}, and Tyler et al. ~\cite{tyler1997procedural} empirically established that individuals evaluate institutions not only by the favorableness of outcomes but also by the perceived fairness of the processes used to generate those outcomes. People are more willing to accept unfavorable decisions when they believe the decision-making process was fair, a finding replicated across legal, medical, and economic contexts. Rawls~\cite{rawls197theory} articulated this philosophically in the concept of pure procedural justice, which stipulates that a just outcome is one that results from a just process, even if its content cannot be independently specified.

In machine learning, these philosophical commitments translate into the requirement that algorithmic decision systems evaluate individuals on consistent, non-discriminatory criteria. Grgić-Hlača et al.~\cite{grgic2018beyond}, Dodge et al \cite{dodge2019explaining}, and Beutel et al \cite{beutel2019putting} were among the first to operationalize this formally, drawing on crowd-sourced human judgments to identify features deemed morally permissible or impermissible as inputs to criminal justice risk assessment tools. Their framework showed that public conceptions of procedural fairness are more nuanced than simple protected-attribute exclusion: participants distinguished between features that causally reflect criminal behavior and those that merely correlate with demographic characteristics. Rueda et al.~\cite{rueda2024just} extend this to medical contexts, arguing from bioethics that procedural fairness in resource allocation demands explainability as a right, grounding the demand for explanation fairness in normative rather than merely technical terms.  Table \ref{tab:procedural} shows common procedural fairness metrics.

\subsection{Ratio-Based and Value-Based Explanation Fairness ($\Delta$REF, $\Delta$VEF)}

Zhao, Wang, and Derr~\cite{zhao2023fairness} introduced two complementary metrics for measuring procedural fairness through the lens of explanation quality. Their framework distinguishes between the bias encoded in model predictions (outcome-oriented bias) and the bias encoded in the decision-making procedure (procedure-oriented bias), arguing that mitigating the latter is both independent of and complementary to mitigating the former.

Let $q: \mathcal{X} \to [0,1]$ be a continuous explanation quality score for instance, the fidelity of the SHAP explanation measured by the decrease in prediction accuracy when top-$k$ features are masked. Let $\tau$ be a quality threshold. The \emph{Ratio-based Explanation Fairness} (REF) metric is:
\begin{equation}
  \Delta\text{REF} = \left| \frac{\sum_{x \in \mathcal{D}_1} \mathbf{1}[q(x) \geq \tau]}{|\mathcal{D}_1|} - \frac{\sum_{x \in \mathcal{D}_0} \mathbf{1}[q(x) \geq \tau]}{|\mathcal{D}_0|} \right|,
  \label{eq:ref}
\end{equation}
directly paralleling the statistical parity definition (\ref{eq:sp}) at the level of explanation quality. The \emph{Value-based Explanation Fairness} metric avoids the discretizations artifact of $\Delta\text{REF}$ by operating on continuous scores:
\begin{equation}
  \Delta\text{VEF} = \left| \mathbb{E}_{x \in \mathcal{D}_1}[q(x)] - \mathbb{E}_{x \in \mathcal{D}_0}[q(x)] \right|.
  \label{eq:vef}
\end{equation}

The CFA algorithm optimizes these metrics alongside traditional fairness by minimizing representation distances between groups in an embedding space. An encoder $f_{\theta_f}$ maps inputs to representations $h(x) = f_{\theta_f}(x)$, and a classifier $c_{\theta_c}$ produces predictions from $h(x)$. For each instance $x_i$, an explainer $E$ computes feature attributions and generates a binary mask $m_i \in \{0,1\}^d$ that zeros out the top-$k$ most important features, yielding a masked input $x_i^m = x_i \odot m_i$. The masked representation is $h^m(x_i) = f_{\theta_f}(x_i^m)$. A key design insight is that the explanation fairness loss \emph{subsumes} the traditional group fairness loss: encouraging representations to be group-invariant both before and after masking simultaneously enforces outcome fairness and explanation fairness. The CFA loss is therefore:
\begin{equation}
  \mathcal{L}_\text{CFA} = \mathcal{L}_u + \lambda\, \mathcal{L}_\text{exp}, \quad \text{where } \; \mathcal{L}_\text{exp} = \mathbb{E}_{y \in \mathcal{Y}}\, \mathbb{E}_{(i,j) \in G_0^y \times G_1^y} \left[D(h_i^y,\, h_j^y) + D(h_i^{m,y},\, h_j^{m,y})\right],
  \label{eq:cfa}
\end{equation}
where $G_s^y = \{i : s_i = s,\, y_i = y\}$ denotes the set of instances from protected group $s$ with utility label $y$, $D(\cdot, \cdot)$ is a distance metric (e.g., MSE, sliced Wasserstein), and $\lambda$ balances utility against fairness. The label-conditioning is critical: it ensures that representations are compared only among instances sharing the same outcome, paralleling the label-conditional structure of Equalized Odds. Because $\mathcal{L}_\text{exp}$ contains both the original-feature term $D(h_i^y, h_j^y)$ (which corresponds to traditional group fairness) and the masked-feature term $D(h_i^{m,y}, h_j^{m,y})$ (which targets explanation fairness), a single loss term addresses both objectives simultaneously.

The key insight is that $\mathcal{L}_\text{explanation}$ minimizes the distance in the explanation-masked embedding space, directly reducing $\Delta\text{VEF}$ while $\mathcal{L}_\text{group}$ reduces outcome-oriented bias. \cite{zhao2023fairness} shows that these objectives are mutually reinforcing, which means optimizing for explanation fairness improves traditional fairness, and vice versa, suggesting a positive relationship rather than a tradeoff.

\subsection{Group Procedural Fairness via Feature Attribution (GPF$_\text{FAE}$)}

Wang, Huang, and Yao~\cite{wang2024procedural} provide a rigorous formalization of group procedural fairness. They define \emph{individual procedural fairness} as the requirement that similar data points from different groups have similar decision-making processes as captured by their feature attribution explanations:
\begin{equation}
  d_e\!\left(g(f_\theta, x^{(i)}),\, g(f_\theta, x^{(j)})\right) \approx 0, \quad \forall x^{(i)}, x^{(j)} \text{ s.t. } d_\mathcal{X}(x^{(i)}, x^{(j)}) \leq \delta,\; A(x^{(i)}) \neq A(x^{(j)}),
  \label{eq:ipf}
\end{equation}
where $g : (f_\theta, x) \mapsto e \in \mathbb{R}^d$ is the FAE function returning feature importance scores and $d_e$ measures distance in explanation space. \emph{Group procedural fairness} is the population-level version requiring that the distribution of explanation vectors for matched pairs from the privileged and unprivileged groups be indistinguishable. The $\text{GPF}_\text{FAE}$ metric formalizes this using Maximum Mean Discrepancy:
\begin{equation}
  \text{GPF}_\text{FAE} = \text{MMD}_k\!\left(\{e^{(i)}\}_{x^{(i)} \in \mathcal{D}_1'},\; \{e^{(j)}\}_{x^{(j)} \in \mathcal{D}_0'}\right),
  \label{eq:gpf}
\end{equation}
where $\mathcal{D}_1'$ and $\mathcal{D}_0'$ are matched datasets of similar pairs, and the MMD is computed with an exponential kernel. A permutation test on the kernel matrix provides a $p$-value.

\cite{wang2024procedural} validate GPF$_\text{FAE}$ on eight real-world datasets (COMPAS, Adult, German Credit, Bank Marketing, Credit Card, OULAD, Lawschool, Diabetes) and demonstrate that outcome-oriented fairness and procedural-oriented fairness are genuinely independent, which means models can be outcome fair and procedurally unfair, outcome unfair and procedurally fair, or both fair or both unfair simultaneously. They propose two mitigation strategies: removing features identified as sources of procedural unfairness, and regularizing the model's weight on the sensitive attribute during training. A related follow-up~\cite{wang2025procedural} extends GPF$_\text{FAE}$ to a training-time regularization framework.

\subsection{Explanation Difference (EDiff)}

Germino et al~\cite{germino2025ediff} introduce \emph{Explanation Difference} (EDiff), which bridges procedural and distributional fairness through a counterfactual lens. The key innovation is the use of counterfactual pairs: rather than comparing different individuals across groups (which may conflate legitimate group differences with procedural discrimination), EDiff compares the \emph{same} individual with their protected attribute flipped, exposing the pure effect of the protected attribute on the explanation. For an explainer $\phi$ (instantiated as FastSHAP~\cite{jethani2021fastshap}), the EDiff of a model $f$ is:
\begin{equation}
  \text{EDiff}(f, X, A, \phi) = \sum_{j=1}^d \frac{1}{|X|} \sum_{i=1}^{|X|} \big|\phi(x_i \mid A=0)_j - \phi(x_i \mid A=1)_j\big|,
  \label{eq:ediff}
\end{equation}
where $\phi(x_i \mid A=a)_j$ is the feature importance attributed to dimension $j$ when the protected attribute is set to $a$. The ideal value is 0. This operationalizes the formal procedural fairness condition: for any two individuals $x_1, x_2$ identical except for protected attribute $A$,
\begin{equation}
    f(x_1) \approx f(x_2) \quad \text{and} \quad |\phi(x_1) - \phi(x_2)| \approx 0,
\end{equation}
where $\Phi$ is the explanation function. The first condition is standard counterfactual fairness, while the second is the procedural fairness extension. A model is equitable when it achieves fair outcomes through a fair process.

The EDiff-based multi-objective optimization loss is:
\begin{equation}
  \mathcal{L}_\text{total} = \mathcal{L}_{F_1} + \mathcal{L}_{SP} + \mathcal{L}_\text{EDiff},
\end{equation}
where $\mathcal{L}_{F_1}$ uses AnyLoss~\cite{han2024anyloss} for differentiable F1 optimization, $\mathcal{L}_{SP} = |P(\hat{y}=1|A=0) - P(\hat{y}=1|A=1)|$ is the statistical parity loss, and $\mathcal{L}_\text{EDiff}$ is as in \ref{eq:ediff}. Curriculum learning is used to compute $\mathcal{L}_\text{EDiff}$ on progressively larger subsets (5\%, 10\%, 20\%, 50\% of training data) to manage computational cost. Experiments on nine datasets show that the EDiff-based approach achieves the best Pareto-optimal trade-off across all possible user preference weightings 59\% of the time, nearly 3.5$\times$ better than the second-best method.

\subsection{Group-Level Explanation Stability Disparity (GESD) and FEU}\label{sec:gesd}
Popoola and Sheppard~\cite{popoola2025gesd} introduce \emph{Group-level Explanation Stability Disparity} (GESD), a complementary metric measuring cross-group variation in explanation \emph{stability under input perturbations}. The key insight is that a group may receive explanations that are on average similar to another group's, yet whose explanations are far more unstable, collapsing or inverting under minor input changes. This robustness-oriented dimension of procedural fairness is invisible to metrics comparing mean attributions.

GESD aggregates SHAP and LIME explanations into an ensemble vector $E_\text{agg}(f,x) = \frac{1}{2}(E_\text{SHAP}(f,x) + E_\text{LIME}(f,x))$. Aggregation is motivated by the empirical finding that SHAP and LIME individually can fail to detect disparities that become visible in their combination \cite{bhatt2020evaluating, mhasawade2024understanding}, a result confirmed in GESD's Recidivism experiments, where SHAP alone yields no significant group difference ($p = 0.96$), but the combined explainer reveals strong disparity ($p = 0.002$). For each instance $x$, $N$ perturbed variants $\tilde{x}$ are generated by combining Gaussian noise and feature masking, and the per-instance instability is:
\begin{equation}
  \Delta(x) = \mathbb{E}_{\tilde{x} \sim P_x}\!\left[\|E_\text{agg}(f,x) - E_\text{agg}(f,\tilde{x})\|_1\right].
\end{equation}
A stability score $S(x) = 1/(1 + \Delta(x))$ inverts this to a higher-is-better quantity. Group stability $S_i$ is the mean stability across all instances in group $G_i$. For binary-protected attributes, GESD is the absolute difference $|S_0 - S_1|$. For $K > 2$ groups, GESD is the variance of group stability across groups.

GESD is model and explainer-agnostic, requiring only a trained model and a post-hoc explanation method. It is therefore applicable as an audit metric without modifying the training procedure, complementing in-processing methods like GCIG that enforce procedural fairness at training time.

\textbf{FEU} (Fairness-Explainability-Utility) integrates GESD into a multi-objective evolutionary optimization framework using NSGA-II, searching over model hyperparameters and decision thresholds to approximate the Pareto front among AUC, GESD, and Demographic Parity. Chebyshev scalarization is used to select a balanced solution from the Pareto front. On the German Credit dataset, FEU generates non-dominated solutions across all three objectives, while baseline methods (Adversarial, Reductions, Reweighting) consistently produce dominated solutions that sacrifice utility for marginal fairness gains. The GESD results also validate the three-objective framing, in which models achieving identical or superior outcome fairness can exhibit substantially different procedural fairness, confirming that GESD captures a genuinely distinct dimension of model behavior.

\subsection{Group Counterfactual Integrated Gradients (GCIG) and FairX}
\label{sec:gcig}

A key limitation shared by prior procedural fairness metrics, including $\Delta$REF, GPF$_\text{FAE}$, GESD, and EDiff, is that they are applied post-hoc, which means the model is trained first, and explanation disparity is measured afterward. This diagnostic approach can reveal procedural inequity but cannot prevent it from arising during training. Popoola and Sheppard~\cite{popoola2026gcig} address this gap with \emph{Group Counterfactual Integrated Gradients} (GCIG), an in-processing regularization framework that enforces explanation invariance across groups as a training-time objective.

The GCIG construction rests on \emph{group conditional baselines}: for each label $y \in \{0,1\}$ and protected group $g \in \mathcal{A}$, a baseline is defined as the expected feature vector of that group:
\[
 b_{y,g} = \mathbb{E}_{X \sim P(X \mid Y=y,\, A=g)}[X] \in \mathbb{R}^p,
\]
estimated in practice via exponentially moving averages over mini-batches to ensure stability under class imbalance. For an input $x$ with true label $y$, the \emph{Group Counterfactual IG attribution} with respect to group $g$ is:
\[
  \mathrm{IG}^{(g)}(x;\, y) = \mathrm{IG}(x,\, b_{y,g}) \in \mathbb{R}^p,
\]
i.e.\ the standard Integrated Gradients attribution computed relative to the group-conditional baseline rather than a global reference. Crucially, the input $x$ is held fixed across groups; only the reference context varies. This operationalizes the counterfactual question: \emph{How would the model's explanation change if the same individual were evaluated relative to a different group context?}

After $\ell_2$-normalization of each attribution vector, the instance-level explanation disparity is:
\begin{equation}
  V(x;\, y) = \left\|\operatorname{Var}_{g \in \mathcal{A}}\!\left[\widehat{\mathrm{IG}}^{(g)}(x;\, y)\right]\right\|_2,
\end{equation}
which, for a binary protected attribute ($|\mathcal{A}|=2$), is proportional to the squared $\ell_2$ distance between the two normalized attribution vectors (since the elementwise variance of two values reduces to a scaled squared difference). The population-level GCIG loss aggregates across labels and instances:
\begin{equation}
  \mathcal{L}_\text{GCIG}(\theta) = \sum_{y \in \{0,1\}} \pi_y\, \mathbb{E}_{X \sim P(X \mid Y=y)}\!\left[V(X;\, y)\right],
\end{equation}
where $\pi_y = P(Y=y)$ are label prevalence weights. This label-stratified construction ensures procedural fairness is enforced separately within each outcome class, directly paralleling the label-conditional structure of Equalized Odds.

In the linear case ($f_\theta(x) = w^\top x + b$), the GCIG loss has a closed form: $\mathcal{L}^{(y)}_\text{GCIG}(\theta) = \|w \odot (b_{y,1} - b_{y,0})\|_2^2$. This shows that GCIG penalizes models that rely on features whose typical values differ across groups, conditional on the label, exactly the directions that produce group-dependent reasoning without legitimate predictive justification.

GCIG is incorporated into a multi-objective training framework called \textbf{FairX}, whose total loss combines prediction accuracy, GCIG procedural fairness, and an outcome-fairness penalty:
\begin{equation}
  \mathcal{L}_\text{total}(\theta) = \mathcal{L}_\text{pred}(\theta) + \lambda_\text{ig}\,\mathcal{L}_\text{GCIG}(\theta) + \lambda_\text{fair}\,\mathcal{L}_\text{fair}(\theta).
\end{equation}
Experiments across four benchmark datasets (Adult Income, German Credit, COMPAS, Bank Marketing) show that FairX reduces explanation disparity by 65--83\% relative to unconstrained training while maintaining competitive F1 and EO gap. Crucially, an ablation study on COMPAS shows that adding outcome fairness alone \emph{worsens} procedural fairness by 14\%, while the joint objective achieves a 24\% improvement, demonstrating a relationship rather than conflict between the two objectives. 

The GCIG framework provides a constructive realization of Axiom 3 (Distributional Parity) from the framework in Section~\ref{sec:axioms}, by penalizing cross-group variance in attribution vectors computed relative to different group-conditional baselines, FairX enforces that the explanation distribution is invariant to group context at training time. Note that GCIG varies the \emph{reference point} (baseline) while holding the input fixed, which tests sensitivity to group-level distributional context rather than individual-level counterfactual attribute change. GCIG is, to our knowledge, the first in-processing method to jointly operationalize Equalized Odds and procedural fairness within a differentiable training objective.

\begin{table}[t]
\caption{Procedural fairness metrics based on feature attribution explanations.}
\label{tab:procedural}
\small
\begin{tabular}{p{2.0cm} p{2.0cm} p{2.5cm} p{2.5cm} p{1.2cm}}
\toprule
\textbf{Metric} & \textbf{Reference} & \textbf{Explanation basis} & \textbf{Comparison mode} & \textbf{Year} \\
\midrule
$\Delta\text{REF}$ & Zhao et al.~\cite{zhao2023fairness} & Quality (binary) & Group averages & 2023 \\
$\Delta\text{VEF}$ & Zhao et al.~\cite{zhao2023fairness} & Quality (continuous) & Group averages & 2023 \\
GPF$_\text{FAE}$ & Wang et al.~\cite{wang2024procedural} & SHAP/gradient & MMD on distributions & 2024 \\
EDiff & Germino et al.~\cite{germino2025ediff} & FastSHAP & Counterfactual pairs & 2025 \\
GCIG & Popoola \& Sheppard~\cite{popoola2026gcig} & Integrated Gradients & Group-conditional baselines & 2026 \\
GESD & Popoola \& Sheppard~\cite{popoola2025gesd} & SHAP + LIME ensemble & Stability under perturbation & 2025 \\
Feature sel.\ & Grgić-Hlača et al.~\cite{grgic2018beyond} & Permissibility & Human-in-the-loop & 2018 \\
Eq.\ explainability & Wang \& Wu~\cite{wang2024achieving} & SHAP & Wasserstein distance & 2024 \\
\bottomrule
\end{tabular}
\end{table}

\subsection{Procedural Bias: Sources and Implications}

The literature identifies three main sources of procedural bias. \emph{Data bias} arises when training data reflects historical discrimination, leading the model to learn different decision criteria for different groups even when protected attributes are excluded~\cite{barocas2017fairness, felin2021data, subramanian2021fairness}. Proxy features then carry discriminatory information into the decision process, creating procedural unfairness that is invisible to feature-selection-based approaches. \emph{Model bias} arises from the architecture and training procedure. In some cases, certain model classes may inherently learn more variable representations for minority groups due to their smaller sample sizes, imbalanced classes, and distributional heterogeneity ~\cite{dai2021say, kordzadeh2022algorithmic, panch2019artificial}. \emph{Explanation bias} is intrinsic to explanation methods. Explainers such as  SHAP can marginalize overcorrelated features, LIME's locality kernel's sensitivity to distribution shift, and GradCAM's focus on high-activation regions can all introduce systematic distortions that disproportionately affect demographic minorities \cite{zhang2019should, lee2019developing}.

The interaction of these three sources creates compounding risks. A model may be procedurally biased due to data bias, explanation methods may fail to detect this bias due to explanation bias, and adversaries may exploit explanation bias to construct fairwashing attacks. Understanding these interactions is a key open problem, discussed further in Section~\ref{sec:open}.

\section{Multi-Objective Optimization for Joint Fairness} \label{sec:multiobj}
A consistent theme across the fairness and XAI literature is that optimization targets are inherently multi-objective. Accuracy, outcome fairness, calibration, and procedural fairness are each legitimate objectives, none reducible to the others~\cite{chouldechova2017fair,kleinberg2016inherent}. This section surveys the multi-objective optimization (MOO) frameworks that have been developed to navigate these tradeoffs, treating the problem as field-wide rather than tied to any single method. Table \ref{tab:multiobj} provides a taxonomy of surveyed methods.

\subsection{The Multi-Objective Perspective on Algorithmic Fairness}

Multi-objective optimization (MOO) refers to finding solutions that optimize multiple potentially conflicting objectives simultaneously~\cite{gunantara2018review}. In the algorithmic fairness context, this perspective predates recent procedural fairness work: Hardt et al.~\cite{hardt2016equality} observed the tradeoff between equal opportunity and accuracy; Chouldechova~\cite{chouldechova2017fair} and Kleinberg et al.~\cite{kleinberg2016inherent} proved impossibility results showing that multiple fairness criteria cannot all be satisfied simultaneously when base rates differ. Agarwal et al.~\cite{agarwal2018reductions} reformulated the entire fair classification problem as constrained optimization. Wei and Niethammer~\cite{wei2022fairness} derived the Pareto frontier between fairness and accuracy. Wick and Tristan~\cite{wick2019unlocking} showed that the tradeoff can be substantially reduced with appropriate modeling choices. Maity et al.~\cite{maity2020notrade} demonstrate conditions under which enforcing fairness can even \emph{improve} accuracy by preventing overfitting to spurious correlations.

The extension to procedural fairness adds a third axis to this established two-dimensional tradeoff. Germino et al.~\cite{germino2025ediff} and Zhao et al.~\cite{zhao2023fairness} show empirically that a fully equitable model must achieve fair outcomes through a fair process while maintaining utility, introducing a three-objective problem: maximize $F_1$ (utility), minimize $|\Delta SP|$ (distributional fairness), and minimize EDiff/REF (procedural fairness). Neither achieving any two of these objectives guarantees the third, nor do optimal solutions to pairs of these objectives necessarily conflict. This three-objective framing is best understood as a \emph{field-level empirical finding} extending Chouldechova's impossibility~\cite{chouldechova2017fair} to three-dimensional constraint geometry, placing procedural fairness on equal footing with distributional concerns.

\subsection{Pareto Fronts and Scalarization in Fairness}

MOO solutions are typically obtained by one of two approaches~\cite{tian2021evolutionary,sharma2022comprehensive, pham2024fairness}. \emph{Pareto-based methods} search for the \emph{Pareto front} which are the set of solutions for which no other solution is superior on every objective simultaneously. Example of Pareto-based MOO algorithms include Non-dominated Sorting Genetic Algorithm (NSGA-II/III) \cite{deb2002fast, ibrahim2016elitensga} and  Wei and Niethammer and Kozdoba et al~\cite{wei2022fairness, kozdoba2024efficient} derive the fairness-accuracy Pareto front for binary classifiers and show that the achievable Pareto frontier depends critically on the choice of fairness metric. Evolutionary and genetic algorithms have been applied to multi-objective fairness problems due to their ability to explore non-convex Pareto fronts~\cite{wang2024generating}. Genetic fairness methods (e.g., FairMask~\cite{peng2022fairmask}) model-based rebalancing of protected attributes using evolutionary search. However, evolutionary methods scale poorly to large neural networks, motivating gradient-based scalarization as the dominant approach in deep learning fairness~\cite{sharma2022comprehensive}.

\emph{Scalarization} converts the multi-objective problem into a single-objective one by forming a weighted combination. An example of scalarization approach for MOO problem is Linear Scalarization \cite{noghin2015linear} which can be express as $\mathcal{L} = \sum_k w_k \mathcal{L}_k$. Germino et al.~\cite{germino2025ediff} use equal-weight scalarisation ($w_k = 1/3$ for three objectives) and demonstrate empirically that this simple scheme achieves a good trade-off across all user preference configurations.

\subsection{AnyLoss: Differentiable Fairness Metric Optimization}

A practical challenge in multi-objective fairness optimization is that common fairness metrics are non-differentiable, preventing their direct use as training objectives in gradient-based learning \cite{e2022differentiable, buyl2023fairret, foulds2019differential}. Han, Moniz, and Chawla~\cite{han2024anyloss} introduce AnyLoss, a framework for transforming any confusion-matrix-based performance measure into a differentiable loss function. The key mechanism is an amplification step that pushes predicted values $\hat{y}_i$ towards 0 or 1:
\begin{equation}
  \tilde{y}_i = \frac{1}{1 + e^{-L(\hat{y}_i - 0.5)}},
\end{equation}
where $L$ is an amplification hyperparameter. The amplified values $\tilde{y}_i$ can then be inserted into standard confusion matrix formulas to produce differentiable losses. The AnyLoss $F_1$ loss is:
\begin{equation}
  \mathcal{L}_{F_1} = 1 - \frac{2 \sum_i y_i \tilde{y}_i}{\sum_i y_i + \sum_i \tilde{y}_i},
\end{equation}
which is differentiable with respect to model parameters $\theta$ and can be combined with fairness losses in a scalarized objective. Germino et al.~\cite{germino2025ediff} use AnyLoss within their EDiff framework to maintain utility while optimizing for both distributional and procedural fairness.

\subsection{Joint Fairness and Explanation Algorithms}

The Comprehensive Fairness Algorithm (CFA)~\cite{zhao2023fairness} jointly optimizes traditional group fairness, explanation fairness ($\Delta$REF, $\Delta$VEF), and predictive utility. CFA operates on learned embeddings: a GNN or MLP encoder maps inputs to a representation $h(x) \in \mathbb{R}^k$, and a decoder produces predictions from $h(x)$. The CFA loss combines three terms: (1) a task loss $\mathcal{L}_\text{task}$ for predictive utility, (2) a group fairness loss $\mathcal{L}_\text{group}$ that penalises the mean distance between the group-conditional embedding centroids $\bar{h}_0$ and $\bar{h}_1$ (not all pairwise distances), encouraging the embeddings to be group-invariant, and (3) an explanation fairness loss $\mathcal{L}_\text{expl}$ that measures $\Delta$REF and $\Delta$VEF on the post-hoc explanations generated from the learned representations. The key insight is that making the representation space group-invariant at the embedding level simultaneously reduces both outcome and explanation disparity, because post-hoc explanations (SHAP, LIME) computed on top of fair embeddings inherit their invariance properties. CFA thus provides empirical evidence that outcome and procedural fairness objectives can be mutually reinforcing when mediated through a shared representation---though this complementarity is an empirical finding specific to their architecture, not a general theoretical guarantee, since Wang et al.~\cite{wang2024procedural} demonstrate that the two objectives can be genuinely independent in other settings.

Wang et al. \cite{wang2024generating} introduce a framework called Generating Explanation for Bias (GEB), which jointly optimizes traditional outcome fairness, a procedure-oriented fairness metric, and utility metrics in a graph neural network (GNN) to retrieve subgraphs that exhibit higher outcome and procedural fairness. The goal of their optimization is to lower the amount of bias that passes through each node in GNNs. 

Fan et al \cite{fan2022explanation} proposed Explanation Guided Genetic algorithm ExpGA, a model-agnostic approach for detecting individual fairness issues in artificial intelligence systems through a combination of explanation methods and genetic algorithms. Their work addresses critical limitations in existing fairness testing approaches, which suffer from low efficiency, low effectiveness, and model-specificity constraints.

Beyond a multi-objective optimization perspective, several works have examined the role of fairness in explanation, particularly how people perceive a prediction's fairness when an explanation is provided. \cite{deck2024critical, ramachandranpillai2025fairxai, matta2023trustworthy, angerschmid2022fairness}

\subsection{Fairness Optimization}

Agarwal et al.~\cite{agarwal2018reductions} provide a general reduction framework that converts a fairness-constrained classification problem into a series of cost-sensitive classification problems, each solvable by a standard learner. Their \emph{exponentiated gradient} algorithm provably converges to the best randomized classifier subject to a given fairness constraint, and can be applied to any fairness metric expressible as a linear constraint on the confusion matrix. Zhang, Lemoine, and Mitchell~\cite{zhang2018mitigating} apply adversarial learning to fairness optimization, training a predictor to maximize utility while an adversary attempts to predict the protected attribute from the predictor's representation. The Nash equilibrium of this game produces representations that are approximately fair.

Hardt, Price, and Srebro~\cite{hardt2016equality} propose a post-processing approach that adjusts classifier thresholds per group to achieve equalized odds, without modifying the underlying model. Feldman et al.~\cite{feldman2015certifying} use pre-processing to remove disparate impact from training data before model training. Peng, Chakraborty, and Menzies~\cite{peng2022fairmask} propose FairMask, which relabels protected attributes in test data using model-based rebalancing, achieving fairness improvements across multiple metrics simultaneously.

FairMOE~\cite{germino2024fairmoe} proposes a mixture-of-experts framework that achieves group fairness by incorporating counterfactual fairness within the optimization process, demonstrating that group and individual fairness can be jointly optimized in a single mixture model.

\subsection{The Fairness-Accuracy-Explainability Trilemma}

Germino et al.~\cite{germino2025ediff} and Schoeffer et al \cite{schoeffer2024explanations} articulate a broader trilemma which states that a model cannot simultaneously maximize (1) predictive performance, (2) distributional fairness, and (3) procedural fairness without compromise. Their experimental results across several datasets show that optimizing any two objectives without the third leads to violation of the third in the majority of cases. This trilemma extends the well-known fairness-accuracy tradeoff~\cite{zliobaite2015accuracy} and the impossibility results of Chouldechova~\cite{chouldechova2017fair} to a three-dimensional setting. Practitioners must make explicit value judgments about how to weight these three objectives, and the choice of weights encodes normative commitments about equity that should be made transparent to stakeholders.

\begin{table}[t]
\caption{Multi-objective fairness optimization frameworks.}
\label{tab:multiobj}
\small
\begin{tabular}{p{2.4cm} p{1.8cm} p{2.5cm} p{2.2cm} p{1.2cm}}
\toprule
\textbf{Method} & \textbf{Reference} & \textbf{Objectives} & \textbf{Strategy} & \textbf{Year} \\
\midrule
EDiff & Germino et al.~\cite{germino2025ediff} & $F_1$, SP, EDiff & Scalarisation & 2025 \\
CFA & Zhao et al.~\cite{zhao2023fairness} & Task, group, expl. & Embedding reg. & 2023 \\
AnyLoss & Han et al.~\cite{han2024anyloss} & Any CM metric & Differentiable loss & 2024 \\
Agarwal et al.\ & \cite{agarwal2018reductions} & Utility, fairness & Reductions & 2018 \\
FairMask & Peng et al.~\cite{peng2022fairmask} & Multiple & Relabelling & 2022 \\
Adversarial Deb. & Zhang et al.~\cite{zhang2018mitigating} & Utility, fairness & Adversarial & 2018 \\
FairMOE & Germino et al.~\cite{germino2024fairmoe} & Group, counterfactual & Mixture of experts & 2024 \\
GEB & Wang et al. \cite{wang2024generating} & Exp, Fair, Utility & Pareto Fronts & 2024\\
FairX & Popoola \& Sheppard~\cite{popoola2026gcig} & $F_1$, EO, GCIG & In-proc.\ regularisation & 2026 \\
FEU & Popoola \& Sheppard~\cite{popoola2025gesd} & AUC, DP, GESD & NSGA-II & 2025 \\
UEF & Popoola \& Sheppard~\cite{popoola2025mesd} & AUC, DP, MESD & NSGA-II + CVaR & 2025 \\
\bottomrule
\end{tabular}
\end{table}

\section{Fairness of Counterfactual Explanations and Algorithmic Recourse}
\label{sec:cf_recourse}

Counterfactual explanations occupy a special position in the fairness-of-explanation literature because they are simultaneously a form of explanation (answering ``why was I denied, and what would have to change?'') and a form of algorithmic recourse (prescribing actionable changes to achieve a desired outcome) \cite{karimi2021algorithmic, upadhyay2021towards, verma2024counterfactual, goyal2019counterfactual}. The fairness of both dimensions has been extensively studied.

\subsection{Recourse Burden Disparities}

Ustun, Spangher, and Liu~\cite{ustun2019actionable} introduced the notion of \emph{actionable recourse}, which is the ability for an individual to change their predicted outcome through personal actions, subject to constraints on what can realistically be changed. Their study of credit and recidivism prediction showed that certain demographic groups face systematically higher recourse burdens which is that they must change more features, by larger amounts, and at greater personal cost, to achieve the same outcome change as more privileged groups. This disparity arises not from intentional discrimination but from the structure of learned decision boundaries in relation to distributional differences between groups.

Von Kügelgen et al.~\cite{vonkugelgen2022fairness} extended this analysis using a causal framework, distinguishing between the \emph{cost of recourse} (the effort required to achieve outcome change) and the \emph{fairness of the recourse process} (whether the prescribed changes respect causal constraints). They showed that counterfactual fairness in the sense of Kusner et al.~\cite{kusner2017counterfactual} does not guarantee fair recourse costs: a model can be counterfactually fair yet impose asymmetric recourse burdens on different groups due to differences in the causal structure of their feature spaces. This motivates a stricter notion of \emph{fair recourse} that constrains not only predictions but also the actions required to change them \cite{kavouras2023fairness}.

Gupta et al.~\cite{gupta2019equalizing} propose a framework for equalizing recourse costs across demographic groups. Their approach modifies the model training procedure to minimize the variance of expected recourse costs across protected groups, subject to a constraint on predictive accuracy. They show that equalized recourse costs can be achieved with small accuracy losses and that the resulting models are also more fair by traditional distributional metrics, suggesting a positive synergy.

\subsection{Individual Fairness in Counterfactual Explanations}

Individual fairness in counterfactual explanation requires that similar individuals receive similar counterfactual explanations \cite{goethals2024precof}. Sharma et al.~\cite{sharma2020certifai} introduce CERTIFAI (Counterfactual Explanations for Robustness, Transparency, Interpretability and Fairness of AI), which generates counterfactual explanations and measures their fairness by examining whether individuals from the same group receive similar counterfactuals. They define a \emph{counterfactual fairness measure} as the average distance between counterfactual explanations for individuals within the same demographic group, arguing that high intra-group variance in counterfactual explanations indicates individual unfairness.

Mothilal et al.~\cite{mothilal2020explaining} propose DiCE (Diverse Counterfactual Explanations), which generates a set of diverse counterfactuals by adding a diversity penalty to the optimization. DiCE has been widely adopted in recourse literature as a baseline, and its fairness properties have been studied by subsequent work. Artelt et al.~\cite{artelt2021evaluating} evaluate the robustness of counterfactual explanations under input perturbations and find that minority groups receive less robust counterfactuals, their prescribed changes are more sensitive to small input variations, meaning the recourse may become invalid if the individual's features are slightly misrecorded.

\subsection{Multi-Dimensional Fairness in Recourse}

Karimi et al.~\cite{karimi2022survey} provide a comprehensive survey of algorithmic recourse, identifying multiple dimensions along which recourse fairness can be evaluated: cost (effort required), validity (whether the counterfactual achieves the desired outcome), feasibility (whether the prescribed changes are actionable given causal constraints), and stability (whether the recourse remains valid over time). They argue that existing fairness analyses have focused primarily on cost disparities while neglecting feasibility and stability disparities.

Ezzeddine et al.~\cite{ezzeddine2025} examine causal recourse under distributional shift, showing that recourse prescriptions that are feasible at training time may become infeasible when the world changes, and that minority groups who are closer to the decision boundary experience greater temporal instability of recourse. Ferrario and Loi~\cite{ferrario2022explainability} document this temporal fairness problem for SHAP-based explanations (Section~\ref{sec:attribution_fairness}); the analogous problem for recourse was independently documented by Venkatasubramanian and Alfano~\cite{venkatasubramanian2020philosophical}.

%\subsection{FACTS: Frequent and Actionable Counterfactuals for Fairness}

%Malioutov and Varshney~\cite{malioutov2022facts} introduce FACTS (Frequent Actionable Counterfactuals for Fairness Testing), a framework for auditing the fairness of algorithmic recourse across demographic groups. FACTS mines for frequent counterfactual patterns, changes that are applied to many individuals, and tests whether these patterns are equitable across groups. Their approach provides a scalable alternative to individual-level recourse analysis for large-scale auditing. Lakkaraju et al.~\cite{ustun2019actionable} and Kavouras et al \cite{kavouras2023fairness} provide additional frameworks for algorithmic recourse auditing.

\subsection{Causal Recourse and Causal Graph Constraints}

Karimi et al.~\cite{karimi2021algorithmic, karimi2020towards, karimi2020algorithmic}  introduce causal algorithmic recourse, which incorporates the causal structure of the data-generating process into the recourse generation procedure. They argue that recourse prescriptions that violate causal constraints, for example, advising someone to reduce their age, or to increase their education level while holding income fixed, when these are causally linked, are not actionable in reality. Under causal constraints, recourse fairness becomes substantially more difficult to achieve because the causal graph structure may create asymmetries in what actions are available to different groups.

PreCoF~\cite{goethals2024precof} provides a pre-commitment framework for counterfactual fairness in which the model commits to fair counterfactual treatments before observing individual protected attributes. This ``veil of ignorance'' approach, inspired by Rawlsian justice \cite{rawls197theory}, ensures that the model cannot tailor its recourse recommendations to protected-group membership, thereby trading off some recourse optimality for procedural guarantees.

\begin{table}[t]
\caption{Key counterfactual explanation and recourse fairness methods.}
\label{tab:cf}
\small
\begin{tabular}{p{2.4cm} p{1.5cm} p{2.5cm} p{2.2cm} p{1.2cm}}
\toprule
\textbf{Method} & \textbf{Reference} & \textbf{Fairness notion} & \textbf{Strategy} & \textbf{Year} \\
\midrule
Actionable Recourse & Ustun et al.~\cite{ustun2019actionable} & Recourse availability & ILP & 2019 \\
Causal recourse & Von Kügelgen et al.~\cite{vonkugelgen2022fairness} & Fair recourse cost & Causal SCM & 2022 \\
Equalizing recourse & Gupta et al.~\cite{gupta2019equalizing} & Equal recourse cost & Training constr. & 2019 \\
CERTIFAI & Sharma et al.~\cite{sharma2020certifai} & CF fairness measure & Genetic algo. & 2020 \\
DiCE & Mothilal et al.~\cite{mothilal2020explaining} & CF diversity & Diversity reg. & 2020 \\
FACTS & Kavouras et al.~\cite{kavouras2023fairness} & Group CF disparity & Pattern mining & 2022 \\
Causal recourse & Karimi et al.~\cite{karimi2021algorithmic} & Causal feasibility & SCM constr. & 2021 \\
PreCoF & Goethals et al. \cite{goethals2024precof} & Individual CF disparity & Pre-processing & 2024 \\
\bottomrule
\end{tabular}
\end{table}

\section{Fairness in Graph-Structured Explanation Methods}
\label{sec:graph}

Graph Neural Networks (GNNs) have emerged as a powerful method for learning on relational data, with applications in social network analysis, recommender systems, drug discovery, and fraud detection \cite{scarselli2008graph, xu2018powerful}. However, the graph structure introduces unique challenges for both fairness and explainability, as the message-passing mechanism can propagate and amplify biases through neighborhood aggregation \cite{hoang2023mitigating}, and standard feature attribution methods must be extended to account for both node attributes and graph topology . This section surveys the intersection of fairness and explainability specifically for graph-structured data.

\subsection{Bias Amplification in Graph Neural Networks}

GNNs are uniquely susceptible to bias amplification due to \emph{homophily}, the tendency for nodes with similar attributes to be connected~\cite{zheng2024missing}. When protected attributes (race, gender, income) are correlated with social ties, the neighborhood aggregation of GNNs causes sensitive information to leak into node representations even when the protected attribute is excluded from the input feature set~\cite{dai2021say}. This creates a form of structural discrimination that is invisible to feature-level auditing.

Chen et al.~\cite{chen2024fairness} survey fairness-aware GNN methods, introducing a taxonomy of fairness evaluation metrics specific to graph data, which are \emph{graph-level fairness} (statistical parity over node predictions), \emph{neighborhood-level fairness} (equity of neighborhood compositions), \emph{embedding-level fairness} (demographic invariance of learned representations), and \emph{prediction-level fairness} (equalized odds at the node or edge level). Each level corresponds to a different locus of bias and requires distinct intervention strategies.

Zhang et al.~\cite{zhang2024trustworthy} provide a comprehensive survey of trustworthy GNNs, treating privacy, robustness, fairness, and explainability as four interconnected dimensions of trustworthiness. Their analysis identifies fairness and explainability as particularly intertwined in the graph setting: bias in GNN predictions often originates from structural properties of the graph (homophily, degree distribution, community structure) that are difficult to explain using node-level attribution methods alone.

\subsection{GNN Explanation Methods}

GNNExplainer~\cite{ying2019gnnexplainer} is the first general model-agnostic approach for explaining GNN predictions. It identifies a compact subgraph structure and a small subset of node features that maximize the mutual information with the GNN's prediction for a given instance. Formally, GNNExplainer solves:
\begin{equation}
  \max_{G_S} I\left(Y, (G_S, X_S)\right) = H(Y) - H\left(Y \mid G = G_S, X = X_S\right),
\end{equation}
where $G_S$ is a subgraph mask and $X_S$ is a feature mask. PGExplainer~\cite{luo2020parameterized} extends this to a parametric approach that can generate explanations for all instances simultaneously. CF-GNNExplainer~\cite{lucic2022cfgnn} produces counterfactual explanations by identifying the minimal edge removals that change the prediction.

GraphLIME~\cite{huang2022graphlime} extends LIME to graph data by constructing locally faithful linear explanations using the $k$-hop neighborhood of each node. The fairness of GraphLIME explanations has not been systematically studied, representing a gap in the literature. GraphXAI~\cite{agarwal2022explainable} provides a benchmark library for evaluating GNN explanation methods on synthetic and real-world datasets with ground-truth explanations.

\subsection{Structural Explanation of Bias in GNNs}

Dong et al.~\cite{dong2022referee} propose REFEREE (stRuctural Explanation oF biasEs in gRaph nEural nEtworks), the first framework for post-hoc explanation of bias in GNNs. REFEREE addresses three challenges: (1) a \emph{fairness notion gap}, defining a principled instance-level bias metric; (2) a \emph{usability gap}, providing both bias and fairness explanations; and (3) a \emph{faithfulness gap}, ensuring explanations are faithful to the GNN prediction process. REFEREE designs two complementary explainers: a \emph{bias explainer} that identifies graph edges contributing most to the exhibited bias of a prediction, and a \emph{fairness explainer} that identifies edges contributing most to fairness. The framework uses a novel instance-level fairness metric derived from counterfactual sensitivity of the node's prediction to changes in its neighborhood.

Dong et al.~\cite{dong2023interpreting} further develop BIND (Bias Influence aNalysis for GNNs via node Deletion), which quantifies to what extent the fairness of a GNN model is influenced by the deletion of individual training nodes without re-training. BIND uses influence function approximations to efficiently estimate the change in fairness metrics ($\Delta SP$, $\Delta EO$) when individual training nodes are removed, thereby enabling targeted identification of training examples disproportionately responsible for model bias.

\subsection{Fairness Explanation in GNN-Based Recommendation}

Deldjoo et al.\ and Ferraro et al.~\cite{ferraro2023gnnuers} introduce GNNUERS (GNN Unfairness Explanation via counterfactual Reasoning for Recommendation), a framework for explaining unfairness in GNN-based recommender systems. Unlike the node-level methods of Dong et al., GNNUERS explains unfairness at the user-item interaction level.  It identifies the minimal set of edges in the interaction graph whose removal would reduce recommendation unfairness to acceptable levels. This provides actionable insights into which behavioral patterns in the training data drive biased recommendations.

The fairness notion in GNNUERS is consumer fairness, whereby users in a protected group (defined by gender, age, or location) should receive recommendations of comparable quality to users in the unprotected group, measured by exposure, relevance, or rank distribution. Explanations are model-level rather than instance-level, reflecting the fact that fairness in recommendation is a system property rather than a property of individual predictions.

\subsection{Fairness Metrics for GNN Explanations}

The fairness of GNN explanations is understudied compared to the fairness of GNN predictions. Key questions include: Do explanation methods (GNNExplainer, PGExplainer) produce explanations of equal quality for nodes from different demographic groups? Are the subgraphs and features identified as important for minority-group nodes as faithful as those for majority-group nodes? Do counterfactual GNN explanations impose equitable recourse costs across demographic groups?

Agarwal et al.~\cite{agarwal2021unified} study the relationship between GNN fairness and explanation quality, finding that fairer GNN models tend to produce more stable and consistent explanations across demographic groups. This positive correlation suggests that fairness mitigation can serve as an implicit regularizer that improves explanation quality, complementing the finding of Zhao et al.~\cite{zhao2023fairness} that explanation fairness and distributional fairness are mutually reinforcing.

The intersection of GNN explainability and intersectional fairness (Section~\ref{sec:intersectional}) is largely unexplored: do GNN explanations exhibit gerrymandering, providing fair explanations on individual attributes but discriminatory explanations on intersectional subgroups? This is identified as an open problem in Section~\ref{sec:open}.

\begin{table}[t]
\caption{Key graph explanation fairness methods.}
\label{tab:graph}
\small
\begin{tabular}{p{2.2cm} p{1.6cm} p{2.5cm} p{2.2cm} p{1.2cm}}
\toprule
\textbf{Method} & \textbf{Reference} & \textbf{Fairness notion} & \textbf{Strategy} & \textbf{Year} \\
\midrule
REFEREE & Dong et al.~\cite{dong2022referee} & Instance-level bias & Bias/fairness explainer & 2022 \\
BIND & Dong et al.~\cite{dong2023interpreting} & Training node influence & Influence functions & 2023 \\
GNNUERS & Ferraro et al.~\cite{ferraro2023gnnuers} & Consumer fairness & CF edge removal & 2023 \\
GNNExplainer & Ying et al.~\cite{ying2019gnnexplainer} & Prediction fidelity & Subgraph masking & 2019 \\
CF-GNNExplainer & Lucic et al.~\cite{lucic2022cfgnn} & CF explanation & Edge removal & 2022 \\
FairGNN survey & Chen et al.~\cite{chen2024fairness} & Multi-level fairness & Survey/taxonomy & 2023 \\
Trustworthy GNNs & Zhang et al.~\cite{zhang2024trustworthy} & Holistic & Survey & 2022 \\
\bottomrule
\end{tabular}
\end{table}

\section{Intersectional Fairness and Fairness Gerrymandering}
\label{sec:intersectional}

Standard fairness analyses fix a small set of protected groups, typically defined by a single demographic attribute, and test for parity of some statistic across these groups. This approach systematically neglects the possibility of discrimination targeting the \emph{intersection} of multiple attributes, which Crenshaw~\cite{crenshaw1989demarginalizing} termed \emph{intersectionality}. A model may satisfy fairness constraints for race and gender independently, yet discriminate severely against their intersection (e.g., Black women) as a subgroup \cite{morina2019auditing}. In the explanation context, this means a model's explanations may be of equal quality across race and gender groups, yet systematically mislead or disadvantage individuals who belong to multiple minority groups simultaneously.

\subsection{Crenshaw's Intersectionality and its ML Translation}

Crenshaw~\cite{crenshaw1989demarginalizing,crenshaw1991mapping} developed intersectionality as a framework for understanding how multiple protected attributes (race, gender, class, disability) interact and compound, producing forms of discrimination that cannot be reduced to any single dimension. Her work showed that anti-discrimination law and policy routinely failed Black women by treating race and gender as separate rather than interacting forms of oppression.

Buolamwini and Gebru~\cite{buolamwini2018gender} provided a landmark empirical demonstration in the ML context, showing that commercial facial recognition systems exhibited error rates up to 34 percentage points higher for darker-skinned women than for lighter-skinned men, a finding that could not have been detected by separate analyses of gender fairness and race fairness. Their methodology for testing intersectional subgroups established an empirical standard for intersectional auditing that has been widely adopted.

Islam et al.~\cite{islam2021can} study \emph{differential fairness} for intersectional subgroups, proposing a generalization of statistical parity that requires $\epsilon$-differential fairness across all subgroups defined by combinations of protected attributes. They show that standard fairness algorithms can satisfy marginal fairness constraints while exhibiting severe intersectional disparities, and that the magnitude of these disparities increases with the number of protected attributes and the degree of their interaction.

\subsection{Fairness Gerrymandering}

Kearns et al.~\cite{kearns2018gerrymandering} formalized the problem of \emph{fairness gerrymandering}, where a classifier can be engineered to appear fair on each of a finite set of pre-defined groups while severely violating fairness on structured subgroups defined over the protected attributes. The term evokes political gerrymandering, in which electoral district boundaries are drawn to produce apparently fair results while creating systematic advantages for one party over another within districts.

Formally, let $\mathcal{C}$ be a rich class of Boolean functions over the protected attribute vector $A = (A_1, \ldots, A_k)$, defining structured subgroups $\{x: c(A(x)) = 1\}$ for $c \in \mathcal{C}$. \emph{Subgroup fairness} requires statistical parity across all $c \in \mathcal{C}$:
\begin{equation}
  \max_{c \in \mathcal{C}} \left| P(\hat{Y} = 1 \mid c(A) = 1) - P(\hat{Y} = 1 \mid c(A) = 0) \right| \leq \epsilon.
\end{equation}
Kearns et al.\ show that auditing subgroup fairness is computationally equivalent to weak agnostic learning, hard in the worst case but tractable using standard ML heuristics. They develop a two-player zero-sum game between a \emph{Learner} (optimizing utility subject to subgroup fairness) and an \emph{Auditor} (finding the most violated subgroup constraint). Their algorithm provably converges to the best subgroup-fair distribution over classifiers in polynomial time.

Kearns et al.~\cite{kearns2019empirical} further show empirically that subgroup fairness violations are pervasive in practice: common fairness algorithms (including reweighting, adversarial debiasing, and post-processing) frequently satisfy marginal fairness constraints while exhibiting substantial violations for intersectional subgroups. This finding directly motivates intersectional fairness auditing as a complement to standard group fairness testing.

\subsection{Intersectional Fairness of Explanations}

The intersection of intersectional fairness with explanation fairness is a particularly underexplored area. Gohar et al.~\cite{gohar2023survey} survey intersectional fairness in machine learning broadly and identify the explanation fairness gap as a critical open problem. While there exist several methods for testing intersectional prediction fairness, only a few systematic methods exist for testing whether explanations are equitable across intersectional subgroups. A model may produce equally faithful SHAP explanations for Black individuals and for women, yet systematically produce less faithful or less actionable explanations for Black women.

Foulds et al.~\cite{foulds2020intersectional} propose differential fairness as a generalization applicable to any number of protected attributes:
\begin{equation}
  \text{DF}(f) = \max_{a, a' \in \mathcal{A}} \left| \ln \frac{P(\hat{Y} = 1 \mid A = a)}{P(\hat{Y} = 1 \mid A = a')} \right| \leq \ln(1 + \epsilon),
\end{equation}
where $\mathcal{A}$ includes all intersectional subgroups defined by combinations of protected attributes. Their differential fairness notion has a natural extension to explanation space: the differential fairness of explanations would require that the distribution of explanation vectors be approximately equal across all intersectional subgroups.

Individual fairness gerrymandering has been studied by Raz ~\cite{raz2024gerrymandering}, who shows that individual fairness, requiring that similar individuals be treated similarly, can itself be gerrymandered. The author mentioned that non-expansive transformations of the feature metric can preserve the individual-fairness Lipschitz constraint while concentrating violations in specific subpopulations. This finding suggests that even the most rigorous notion of individual fairness is susceptible to subgroup exploitation under adversarial metric design.

\subsection{Survey of Intersectional Fairness Methods}

Several methods have been proposed to achieve intersectional fairness. Hébert-Johnson et al.~\cite{hebert2018multicalibration} introduce \emph{multicalibration}, which requires that predictions be calibrated for all groups in a rich class $\mathcal{C}$, including intersectional subgroups. Multicalibration is achievable through an efficient boosting-based algorithm and provides strong fairness guarantees for downstream decision-making. Kearns et al.~\cite{kearns2019empirical} provide an empirical study comparing subgroup fairness algorithms on real datasets, showing that game-theoretic approaches based on their auditor-learner framework outperform standard fairness-constrained optimization on intersectional subgroup metrics.

Kim et al.~\cite{kim2019multiaccuracy} propose \emph{multiaccuracy} as a complementary notion requiring accurate predictions across all $\mathcal{C}$-definable subgroups, showing that multiaccuracy implies fair representation in a statistical sense. For fairness in explanation, the analogous notion would be multi-accuracy of explanations: requiring that explanations be faithful (accurately capturing the model's behavior) across all intersectional subgroups in $\mathcal{C}$.

\subsection{Multi-Category Explanation Stability Disparity (MESD) and UEF}
\label{sec:mesd}

Prior procedural fairness metrics operate on a single protected attribute, making them susceptible to fairness gerrymandering, in which a model may achieve low procedural fairness across racial and gender groups while systematically disadvantaging the intersection of both groups. Popoola and Sheppard~\cite{popoola2025mesd} address this gap with \emph{Multi-Category Explanation Stability Disparity} (MESD), the first procedural fairness metric designed explicitly for intersectional subgroups defined by multiple protected attributes.

MESD constructs intersectional subgroups as the Cartesian product of $m$ protected features: $\mathcal{G} = A_1 \times A_2 \times \cdots \times A_m$. Each individual belongs to exactly one subgroup $g(x) = (A_1(x), \ldots, A_m(x))$. Explanation stability per instance is measured as in GESD, via ensemble SHAP+LIME under Gaussian and masking perturbations, but MESD introduces three methodological advances over GESD to handle the statistical challenges that arise with many small intersectional subgroups:

\textbf{(1) Label-aware aggregation.} Rather than comparing raw stability scores across groups, MESD partitions the dataset into subgroup--label cells $\mathcal{D}_{g,y} = \{x: A(x) = g,\, Y(x) = y\}$ and aggregates stability separately within each cell. Group-level stability is then a label-prevalence-weighted average: $S(g) = \sum_{y} w_y \tilde{S}(g,y)$. This mirrors the label-conditional structure of Equalized Odds, ensuring procedural fairness is evaluated at the same granularity as the standard outcome-fairness criterion.

\textbf{(2) Empirical-Bayes shrinkage.} Small intersectional groups produce high-variance stability estimates. MESD applies an Empirical-Bayes shrinkage estimator that pulls small-sample group estimates toward the label-wise mean:
\begin{equation}
  \tilde{S}(g,y) = \alpha_{g,y} S(g,y) + (1 - \alpha_{g,y})\bar{S}(y), \quad \alpha_{g,y} = \frac{|\mathcal{D}_{g,y}|}{|\mathcal{D}_{g,y}| + \lambda},
\end{equation}
where $\lambda > 0$ controls the degree of regularization. This allows MESD to produce stable estimates even for rare intersectional groups such as ``white female'' in the Recidivism dataset.

\textbf{(3) CVaR-weighted tail disparity.} Simple variance-based aggregation of group stabilities can mask severe worst-case disparities. MESD uses Conditional Value-at-Risk (CVaR) to weight pairwise group disparities by the risk severity of the lower-stability group:
\begin{equation}
  \text{MESD}_\text{CVaR}(\alpha) = \sum_{i < j} w_{ij}\, |S(g_i) - S(g_j)|,
\end{equation}
where $w_{ij}$ concentrates weight on pairs whose minimum stability falls below the $(1-\alpha)$-quantile threshold. .

MESD is integrated into an NSGA-II-based multi-objective optimization framework called \textbf{UEF} (Utility-Explainability-Fairness), which jointly minimizes AUC loss, Demographic Parity, and MESD, using Chebyshev scalarization for Pareto front selection. Across three benchmark datasets (German Credit, Adult Income, Recidivism), UEF outperforms Adversarial Learning, Reductions, and Reweighting baselines on MESD on all three datasets and on utility (AUC and F1) on two of three.

The relationship between MESD and the conditional invariance framework (Section~\ref{sec:framework}) is direct. MESD operationalizes a robustness-oriented variant of Axiom 3 (Distributional Parity) at the intersectional level, requiring that explanation stability be equal not only across marginal groups but across all cells of the intersectional subgroup space. By aggregating across labels (Axiom 3, Distributional Parity), MESD also tests whether the invariance condition holds at the population level for each intersectional cell. The CVaR tail formulation also guards against fairness gerrymandering identified by Kearns et al.~\cite{kearns2018gerrymandering}, ensuring that apparent procedural parity at the marginal-group level cannot conceal severe instability in an underrepresented intersection.

\begin{table}[t]
\caption{Intersectional fairness and fairness gerrymandering methods.}
\label{tab:intersectional}
\small
\begin{tabular}{p{2.4cm} p{2.0cm} p{2.5cm} p{1.8cm} p{1.2cm}}
\toprule
\textbf{Method} & \textbf{Reference} & \textbf{Fairness notion} & \textbf{Strategy} & \textbf{Year} \\
\midrule
Gender shades & Buolamwini \& Gebru~\cite{buolamwini2018gender} & Intersectional error & Audit & 2018 \\
Subgroup fairness & Kearns et al.~\cite{kearns2018gerrymandering} & Fairness gerrymandering & Game-theoretic & 2018 \\
Differential fairness & Foulds et al.~\cite{foulds2020intersectional} & $\epsilon$-DF & KL-divergence & 2020 \\
Multicalibration & Hébert-Johnson et al.~\cite{hebert2018multicalibration} & Calibration & Boosting & 2018 \\
Multiaccuracy & Kim et al.~\cite{kim2019multiaccuracy} & Accuracy & Post-processing & 2019 \\
IF gerrymandering & RAZ~\cite{raz2024gerrymandering} & Individual fairness & Metric analysis & 2024 \\
MESD & Popoola \& Sheppard~\cite{popoola2025mesd} & Explanation stability & CVaR + NSGA-II & 2025 \\
\bottomrule
\end{tabular}
\end{table}

\section{Beyond Attributions: Rule, Example, Concept, and Neurosymbolic Methods}\label{sec:beyond}
The fairness-of-explanation literature has been dominated by feature attribution methods, yet a substantial body of work examines the fairness of other XAI paradigms. This section surveys fairness concerns in rule-based, example-based, concept-based, and neurosymbolic explanation methods.

\subsection{Rule-Based Explanations and Fairness}

Rule-based methods produce if-then explanations that are human-readable and actionable. ANCHOR~\cite{ribeiro2018anchors} generates rules of the form ``If \emph{income $>$ 50k} and \emph{education $\geq$ Bachelor's}, then prediction = positive.'' The fairness concerns for such methods include: (1) Do rules have equal precision and coverage across demographic groups? (2) Do rules invoke different features for different groups? (3) Are the rules stable across slight input perturbations?

Kaur et al.~\cite{kaur2020interpreting} study how practitioners interact with rule-based explanations and find significant variability in how individuals from different educational backgrounds interpret and act on these rules, raising concerns about cognitive accessibility fairness. Rudin~\cite{rudin2019stop} argues for intrinsically interpretable models (decision trees, rule lists, scoring systems) in high-stakes domains precisely because they provide guaranteed explanation quality uniformly across all individuals. Letham et al.~\cite{letham2015interpretable} provide Bayesian Rule Lists (BRL), which learn rule sets from data; the fairness of BRL explanations has been studied as a function of the training data composition.

The interaction of rule-based explanations with procedural fairness is noteworthy. Grgić-Hlača et al.~\cite{grgic2018beyond} use feature selection to determine which features appear in procedurally fair rules; features not deemed morally permissible are excluded from the rule set regardless of their predictive value. This creates a tension between predictive utility and procedural fairness that is an instance of the broader fairness-accuracy tradeoff.

\subsection{Example-Based Methods and Fairness}

Prototype and criticism methods~\cite{kim2016examples} explain a model's behavior by identifying representative (\emph{prototypes}) and unrepresentative (\emph{criticisms}) training examples. Influence functions~\cite{koh2017understanding} identify which training examples most influence a given prediction. The fairness of these methods has been studied by examining whether prototypes and criticisms are representative of all demographic groups \cite{wang2024fairif}. A model whose prototypes are drawn predominantly from the majority group provides explanations that may not be informative or relatable for minority-group individuals~\cite{kim2016examples}.
In the context of counterfactual explanations, the choice of exemplar counterfactuals can be unfair if the most ``natural'' counterfactuals identified by the method consistently require minority-group individuals to adopt majority-group attributes \cite{guo2024counterfactual}. This is closely related to the recourse fairness problem (Section~\ref{sec:cf_recourse}).

\subsection{Concept-Based Explanations and Fairness}

TCAV~\cite{kim2018interpretability} measures the sensitivity of a neural network's prediction to user-defined concepts using concept activation vectors. The fairness of TCAV explanations depends on whether the user-defined concepts are representative of and meaningful for all demographic groups. In medical imaging, for example, concepts defined from majority-population medical images may produce less meaningful activation vectors for minority patients whose physiological characteristics differ from the concept examples.

Concept Bottleneck Models (CBMs)~\cite{koh2020concept, chauhan2023interactive} explicitly predict human-defined concepts as intermediate bottleneck variables before making a final prediction. CBMs offer stronger guarantees of explanation faithfulness than post-hoc methods because the concepts are part of the model architecture rather than external approximations. However, the choice of concepts encodes cultural and social assumptions that may disadvantage certain groups, a form of conceptual bias that can propagate to explanation fairness. Also, concept leakage is another issue in concept bottleneck models, where protected attributes can leak into the model prediction \cite{havasi2022addressing}.

Waller et al.~\cite{waller2024identifying} use an argumentation framework to identify which concepts (features) are responsible for observed model bias, producing concept-level explanations of fairness violations. This bridges concept-based XAI with fairness auditing in a principled way.

\subsection{Epistemic Fairness: Explanation as the Ability to Act} \label{sec:epistemic}

A fundamental insight missing from attribution-centric views of explanation fairness is that explanations serve a \emph{communicative} and \emph{enabling} function, not merely an informational one. An explanation is fair not only when it assigns equal importance scores across groups but when it produces \emph{equal ability to understand and act} on those scores. We term this \emph{epistemic fairness}, which is the property that an explanation grants all recipients equivalent epistemic access to the information needed to contest, understand, and respond to an algorithmic decision~\cite{venkatasubramanian2020philosophical}.

Our use of the term \emph{epistemic fairness} here is deliberately narrower than the philosophical concept of \emph{epistemic injustice} introduced by Fricker~\cite{fricker2007epistemic} and applied to algorithmic systems by Edenberg and Wood~\cite{edenberg2023epistemic}. In their framework, epistemic injustice refers to harms done to individuals \emph{in their capacity as knowers}, for example, when algorithmic systems reinforce prejudicial assumptions about who is credible, competent, or trustworthy based on group identity. This upstream concept addresses how algorithms shape the epistemic frameworks through which people are perceived and understood socially, encompassing testimonial injustice (credibility deficits tied to group membership), hermeneutical injustice (gaps in collective interpretive resources), and oppressive epistemic systems that are resistant to reform because they structure how people interpret the world~\cite{dotson2014epistemic, berenstain2022epistemic}. Our concept of epistemic fairness operates downstream, whereby, given that an explanation has been generated, does it grant \emph{equal epistemic access} to all recipients? Both dimensions are necessary because an explanation system can satisfy our functional epistemic fairness criterion while still operating within an oppressive epistemic framework that shapes which features are deemed relevant, which causal stories are told, and whose experiences are legible to the system in the first place.

Epistemic fairness is violated even when attribution parity holds, through at least four mechanisms. \emph{Narrative framing bias}: equivalent feature importance scores may be communicated in framing that disadvantages certain groups, for example, framing a denial as a deficit (``your credit score is insufficient'') versus as an opportunity (``increasing your credit score by 50 points would change this decision'') produces different psychological and practical responses~\cite{winterbottom2008does}. \emph{Recourse feasibility inequity}: even when counterfactual explanations are mathematically equidistant across groups, the prescribed actions may be structurally inaccessible to some groups due to socioeconomic constraints, geographic barriers, or institutional gatekeeping, therefore making the explanation nominally equal but functionally unequal. \emph{Cognitive accessibility disparities}: Lunich and Keller~\cite{lunich2024explainable} demonstrate that explanation complexity interacts with educational background to produce dramatically different comprehension rates, whereby an explanation that is equally complex for all groups is not equally informative. Kaur et al.~\cite{kaur2020interpreting} show that practitioners with different technical backgrounds interpret the same SHAP visualizations in systematically different ways, with non-experts consistently misinterpreting feature importance as causal rather than correlational. \emph{Human-AI interaction asymmetry}: Menon et al.~\cite{menon2024lessons} document that users from marginalized groups are less likely to challenge or contest algorithmic explanations, not because the explanations are less contestable but because social and institutional power dynamics suppress contestation, a phenomenon invisible to technical auditing.

Epistemic fairness thus demands a paradigm shift from measuring whether explanations \emph{are} equal to measuring whether they \emph{function} equally for all recipients. A fifth dimension deserves explicit mention: \emph{trust calibration fairness}, the degree to which different groups develop appropriately calibrated trust in the system's explanations. Over-trust (accepting explanations uncritically) and under-trust (dismissing valid explanations) are both epistemic harms, and research by Angerschmid et al.~\cite{angerschmid2022fairness} shows that trust calibration varies significantly across demographic groups with different prior experiences of algorithmic systems. Groups with historical experiences of discriminatory automation may rationally distrust explanations even when those explanations are accurate, producing a form of epistemic inequity that is not remediable by improving explanation quality alone. This requires moving beyond purely technical metrics toward human-centered evaluation methodologies, co-designed with the communities most affected by algorithmic decisions. Rule-based methods~\cite{ribeiro2018anchors}, concept-based methods~\cite{koh2020concept}, and inherently interpretable models~\cite{rudin2019stop} have structural advantages for epistemic fairness because their explanations are closer to natural language and human reasoning; however, they too can be epistemically unfair if the concepts, rules, or scoring systems encode culturally specific knowledge that is inaccessible to minority group members.

\subsection{Neurosymbolic Methods and Fairness}

Neurosymbolic approaches combine neural learning with symbolic reasoning, enabling the encoding of fairness constraints as logical axioms \cite{bhuyan2024neuro}. Wagner and Garcez~\cite{wagner2021logical} develop Logic Tensor Networks (LTNs) that integrate first-order fuzzy logic into neural learning, enabling the encoding of fairness predicates as logical clauses. A fairness constraint such as ``similar individuals should receive similar predictions'' can be encoded as a universally quantified formula and enforced via real-valued logic satisfaction during training. The advantage of neurosymbolic methods for explanation fairness is that logical constraints are satisfied by design rather than as a post-hoc approximation, providing stronger guarantees than regularization-based approaches. 
%However, this comes at the cost of requiring the fairness constraints to be specified in formal logic, which places a high burden on domain experts and may exclude nuanced social considerations that resist formalization.
Heilmann et al.~\cite{heilmann2026neurosymbolic} extend neurosymbolic reasoning to counterfactual fairness, showing that SCM-based counterfactual fairness constraints can be encoded as symbolic rules within a neurosymbolic system and enforced differentially across demographic subgroups.

\section{Auditing Frameworks for Explanation Fairness}
\label{sec:auditing}

Auditing frameworks provide systematic tools for measuring and certifying explanation fairness. This section surveys key auditing approaches, distinguishing between post-hoc measurement frameworks and active certification systems.

\subsection{Post-Hoc Measurement Frameworks}

The FACTS framework~\cite{kavouras2023fairness} provides a scalable auditing tool based on frequent pattern mining of disparities in counterfactual explanations. AReS (Actionable Recourse Summary)~\cite{ustun2019actionable} provides an auditing language for recourse fairness across demographic groups, identifying which groups have access to actionable recourse and quantifying the cost asymmetries. Aequitas~\cite{saleiro2018aequitas} is a widely used audit toolkit for measuring disparity across multiple fairness metrics simultaneously and has since been extended to include explanation quality metrics.

The GPF$_\text{FAE}$ validation experiments of Wang et al.~\cite{wang2024procedural} demonstrate a systematic auditing methodology, whereby, given a model, apply SHAP to generate feature attribution vectors for all instances, match instances across protected groups by similarity, compute MMD on matched explanation distributions, and apply a permutation test for statistical significance. This methodology provides a principled audit of procedural fairness that can be applied to any model and any post-hoc FAE method.

\subsection{Recommended Evaluation Workflow}
\label{sec:eval_workflow}

A recurring obstacle to progress in explanation fairness research is the absence of a canonical evaluation protocol. Existing work uses inconsistent datasets, preprocessing choices, baseline models, and metrics, making direct comparison across studies difficult. Drawing on the three generative pathways (Section~\ref{sec:mechanisms}), the failure taxonomy (Section~\ref{sec:failure_taxonomy}), and the practitioner decision guide (Section~\ref{sec:decision_guide}), we propose the following six-step standard evaluation workflow for explanation fairness audits. It is designed to probe all three pathways- representation-driven, explanation-model mismatch, and actionability-driven in a structured sequence.

\begin{mdframed}%[backgroundcolor=blue!8]
\textbf{Standard Explanation Fairness Evaluation Workflow}

\medskip
\textbf{Step 1: Outcome fairness screening.} Compute standard outcome-oriented fairness metrics ($\Delta$SP, $\Delta$EO, calibration within groups) as a baseline. Record which metrics pass and fail. This establishes whether any subsequent explanation unfairness is compounded by, or independent of, outcome unfairness.

\medskip
\textbf{Step 2: Explanation distribution test.} Apply any post-hoc method to generate attribution vectors for all instances. Compute GPF$_\text{FAE}$ (MMD on matched pairs) and $\Delta$VEF (mean quality disparity) or GESD across protected groups. A statistically significant gap flags representation-driven or surrogate-mismatch inequity. 

\medskip
\textbf{Step 3: Counterfactual explanation symmetry test.} Compute EDiff using counterfactual pairs, the same individual with the protected attribute flipped. An EDiff $> 0$ indicates procedural unfairness independent of outcome fairness. This step directly probes the conditional invariance condition (Eq.~\ref{eq:ef_invariance}).

\medskip
\textbf{Step 4: Recourse burden comparison.} For counterfactual explanation methods, compute mean recourse cost (Wachter distance, feature-change count, or actionable recourse score~\cite{ustun2019actionable}) per protected group. Test for significant disparity using FACTS~\cite{kavouras2023fairness} or the equalizing recourse framework~\cite{gupta2019equalizing}.

\medskip
\textbf{Step 5: Stability under distributional shift.} Re-evaluate Steps 2-4 on held-out data from a later time period or a shifted distribution (e.g., post-COVID economic data for a credit model). Compute the temporal explanation unfairness $\text{TEF}(t)$ differential across groups to assess whether minority groups experience faster explanation obsolescence.

\medskip
\textbf{Step 6: Human interpretability validation.} Conduct a user study or human evaluation with participants from affected demographic groups. Measure comprehension accuracy (can participants identify the key reason for the decision?), contestation rate (do participants know how to contest?), and perceived fairness. This step probes actionability-driven and epistemic inequity that is invisible to Steps 1--5.
\end{mdframed}

\medskip
Steps 1--4 are fully automatable and should be run as a minimum for any deployment, while Steps 5--6 require longitudinal data and human participants, respectively, and are recommended for high-stakes decisions subject to regulatory scrutiny. A model that fails Steps 1- 2 but passes Steps 3- 4 has outcome bias that leaks into explanations; a model that passes Steps 1- 2 but fails Step 3 is outcome-fair but procedurally unfair, the distinction the procedural fairness framework was designed to expose.

The workflow maps directly to the three generative pathways in section \ref{sec:mechanisms}: Steps 2--3 probe pathway 1 (representation-driven) and pathway 2 (explanation-model mismatch); Step 4 probes pathway 3 (actionability); Steps 5--6 probe the temporal and epistemic dimensions of all three pathways.

\subsection{Failure Taxonomy and Remediation Map}
\label{sec:failure_taxonomy}

Drawing on the conditional invariance framework (Section~\ref{sec:framework}) and the causal mechanisms identified therein, Table~\ref{tab:failure_taxonomy} provides a structured taxonomy of explanation fairness failures, mapping each failure type to its causal origin, observable audit signals, appropriate metrics, and evidence-based remediation strategies. This taxonomy is intended as a practical decision tool for practitioners conducting explanation fairness audits.

\begin{table*}[t]
\caption{Failure taxonomy for explanation fairness auditing. Each row maps a failure type to its causal mechanism, audit signals, measurement metrics, and remediation strategies.}
\label{tab:failure_taxonomy}
\small
\begin{tabular}{p{2.2cm} p{2.4cm} p{2.5cm} p{2.2cm} p{3.0cm}}
\toprule
\textbf{Failure type} & \textbf{Causal mechanism} & \textbf{Audit signal} & \textbf{Metric} & \textbf{Remediation} \\
\midrule
Representation disparity & Training data imbalance & Lower SHAP stability for minority & $\Delta\text{VEF}$, GESD & Resampling, augmentation \\
\addlinespace[2pt]
Surrogate mismatch & Sparse manifold regions & Low LIME fidelity for minority & GCIG & Distribution-aware explainer \\
\addlinespace[2pt]
Proxy leakage & Correlated proxy features & Proxy feature dominates SHAP & EDiff & Causal debiasing \\
\addlinespace[2pt]
Actionability gap & Structural socioeconomic barriers & High recourse cost disparity & Recourse CD & Constraint-aware recourse \\
\addlinespace[2pt]
Fairwashing & Model-agnostic auditing & Explanation-prediction mismatch & Anders impossibility & Ante-hoc interpretable model \\
\addlinespace[2pt]
Temporal degradation & Distributional shift & Explanation obsolescence rate & $\text{TEF}(t)$ & Adaptive recalibration \\
\addlinespace[2pt]
Intersectional blind spot & Gerrymandering & Marginal fairness with subgroup violation & GPF$_\text{FAE}$ by subgroup & Multicalibration, subgroup audit \\
\addlinespace[2pt]
Cognitive accessibility gap & Unequal comprehension & Human evaluation study & User study metrics & Narrative redesign, co-design \\
\bottomrule
\end{tabular}
\end{table*}

\subsection{Practitioner Decision Guide: Choosing an Explanation Fairness Method}
\label{sec:decision_guide}

Choosing among available explanation fairness methods requires matching method assumptions to the deployment context. Table~\ref{tab:decision_guide} provides a structured guide for practitioners, mapping key contextual constraints to recommended methods and their limitations.

\begin{table*}[t]
\caption{Practitioner decision guide for selecting explanation fairness methods based on deployment context and constraints.}
\label{tab:decision_guide}
\small
\begin{tabular}{p{2.8cm} p{2.5cm} p{2.5cm} p{2.5cm} p{2.0cm}}
\toprule
\textbf{Context} & \textbf{Key constraint} & \textbf{Recommended method} & \textbf{Limitation} & \textbf{Reference} \\
\midrule
Post-hoc audit, no model access & Black-box access only & GPF$_\text{FAE}$ (SHAP-based) & Cannot detect fairwashing & \cite{wang2024procedural} \\
\addlinespace[2pt]
Training-time mitigation & Differentiable model & GCIG & Require Neural Network & \cite{popoola2026gcig} \\
\addlinespace[2pt]
Known causal graph & Causal structure available & Causal recourse (Karimi et al.) & Graph must be correct & \cite{karimi2021algorithmic} \\
\addlinespace[2pt]
Recourse focus & Actionability required & FACTS, equalizing recourse & Cost metric dependent & \cite{kavouras2023fairness} \\
\addlinespace[2pt]
Graph-structured data & Node/edge relational data & REFEREE, BIND & GNN-specific & \cite{dong2022referee} \\
\addlinespace[2pt]
Intersectional audit & Multiple protected attributes & Multicalibration + subgroup audit & Exponential subgroups & \cite{hebert2018multicalibration} \\
\addlinespace[2pt]
Human-centred evaluation & Epistemic fairness required & Participatory design + user study & Resource-intensive & \cite{dodge2019explaining} \\
\bottomrule
\end{tabular}
\end{table*}

\subsection{What Explanation Fairness Cannot Guarantee}
\label{sec:limits}

The audit methods described above are useful and necessary. They are also bound in ways that practitioners and regulators must understand explicitly. This subsection states those bounds without hedging.

\textbf{Post-hoc explanation methods cannot certify fairness.} Anders et al.~\cite{anders2020fairwashing} prove that under natural axiomatic conditions, most post-hoc methods can distinguish a fair from an unfair model using explanation output alone. This is not a statement that current tools are insufficiently good, but rather about what \emph{most} post-hoc tools can achieve. A system that passes all post-hoc explanation fairness audits can still be profoundly unfair. Auditors who rely exclusively on explanation inspection are receiving a lower bound on fairness, not a certificate.

\textbf{Explanation fairness $\neq$ model fairness.} A model can produce distributionally fair predictions ($\Delta$SP $\approx 0$) while being procedurally unfair~\cite{zhao2023fairness}. Conversely, a model can be procedurally fair in its explanations while producing discriminatory outcomes. Wang et al.~\cite{wang2024procedural} demonstrate both combinations empirically across eight datasets. Outcome fairness audits do not substitute for explanation fairness audits, and vice versa.

\textbf{Post-hoc methods cannot ensure causal neutrality.} SHAP and LIME distribute attribution credit according to the model's learned correlations, not causal structure~\cite{ng2025causal}. When protected attributes are correlated with proxy features, explanation methods will attribute credit to proxies in proportion to their correlational contribution. The resulting explanation is technically faithful to the model's statistical behavior but causally misleading about the source of discrimination. There is no post-hoc fix for this; it requires causal graph knowledge or model restructuring.

\textbf{Some inequities are epistemic, not statistical.} The cognitive accessibility gap (Section~\ref{sec:epistemic}) and actionability asymmetry (Section~\ref{sec:mechanisms}) are not detectable by any metric over explanation vectors. A model whose SHAP distributions are identical across groups may still systematically disadvantage minority users who cannot effectively interpret or act on those explanations. Epistemic unfairness is invisible to Steps 1- 4 of the evaluation workflow and requires human-centered methods (Step 6) to detect.

\textbf{The identifiability barrier is hard to remove.} Post-hoc explanation fairness is an interventional quantity (Section~\ref{sec:framework}), it asks what the explanation \emph{would be} if the protected attribute were different. Answering this question exactly requires causal assumptions not available from observational data. Every metric in Table~\ref{tab:unification} makes different approximating assumptions to work around this barrier. 
Stating these limits explicitly is not defeatist, but it tells practitioners what they can rely on, what they cannot, and where the boundary lies.

Internal auditing methods that directly inspect model weights, training data, and learned representations offer stronger guarantees than external explanation-based auditing~\cite{rudin2019stop}, but even these cannot fully solve the identifiability problem without causal assumptions.

\subsection{Regulatory Context}

The European Union's General Data Protection Regulation (GDPR), Article 22, establishes a right to an explanation for individuals subject to automated decisions. The EU AI Act (2024) imposes transparency and fairness requirements on high-risk AI systems. In the United States, the Equal Credit Opportunity Act (ECOA) and Fair Housing Act impose non-discrimination requirements in credit and housing, and regulatory guidance from the Consumer Financial Protection Bureau (CFPB) increasingly requires explainability as a component of fair lending compliance. These regulatory frameworks create strong demand for auditable explanation fairness, motivating the technical developments surveyed in this paper.

\section{Domain Applications}\label{sec:domains}
The abstract concerns of explanation fairness become concrete and urgent in specific high-stakes application domains. This section surveys four domains where the intersection of fairness and explainability has been most intensively studied and concludes with a cross-domain synthesis identifying common structural features that explanation fairness frameworks must address.

\subsection{Criminal Justice and Recidivism Prediction}

The COMPAS recidivism prediction system has become the most extensively studied case study in algorithmic fairness. Angwin et al.~\cite{angwin2016machine} showed that COMPAS misclassified Black defendants as high-risk at nearly twice the rate of white defendants. 

From an explanatory perspective, criminal justice prediction raises acute concerns. Slack et al.~\cite{slack2020fooling} used COMPAS as one of their experimental domains, demonstrating that both SHAP and LIME could be fairwashed on recidivism prediction data. Wang et al.~\cite{wang2024procedural} validate GPF$_\text{FAE}$ on COMPAS and show that procedurally unfair models are common even when distributional fairness constraints are satisfied. The actionability of recourse recommendations from COMPAS-like systems is particularly contested; recidivism risk factors may include features that are causally downstream of protected attributes (e.g., prior arrests, which are influenced by racial profiling) and that are therefore not genuinely actionable through individual behavior change. This domain illustrates the failure of Axiom 4 (actionability asymmetry) in its most consequential form.

\subsection{Healthcare and Clinical Decision Support}

Obermeyer et al.~\cite{obermeyer2019dissecting} demonstrated that a widely used commercial healthcare algorithm allocated resources based on healthcare costs rather than healthcare needs, resulting in systematic underallocation to Black patients, who incur lower costs for equivalent health needs due to inequitable access to care. Their study used SHAP decomposition to attribute this disparity to the specific features driving the algorithm's predictions, demonstrating the power of feature attribution methods for diagnosing bias.

From an explanation fairness perspective, healthcare introduces additional complexity, where clinical decisions must often be explained to patients, families, and clinicians in ways that are comprehensible to their level of medical literacy. Menon et al.~\cite{menon2024lessons} draw on clinical communication research to argue that explanation quality in medical AI must account for the diversity of health literacy across demographic groups, requiring group-aware explanation design rather than one-size-fits-all explanation methods. Lunich and Keller~\cite{lunich2024explainable} show that explanation complexity affects perceived fairness in educational prediction contexts, with implications for how clinical AI explanations should be designed for equity. This domain uniquely surfaces Axiom 5 violations (epistemic accessibility), where a patient may receive a technically accurate explanation of a clinical risk score that they cannot meaningfully interpret or contest.

\subsection{Credit and Lending}

The credit domain provides the richest empirical context for the fairness of explanation. Fuster et al.~\cite{fuster2022predictably} document that machine learning-based credit scoring algorithms produce predictions that are more accurate on average but exhibit higher racial disparities than traditional scoring methods. This ``accuracy-fairness'' tradeoff is exacerbated by the fact that ML models use a richer feature set, which may include proxies for race that are not explicitly excluded \cite{robb2018testing}.

Ustun et al.~\cite{ustun2019actionable} document recourse disparities in lending, showing that individuals with lower income, fewer education credentials, and higher debt, characteristics disproportionately prevalent in minority groups, face systematically higher barriers to algorithmic recourse. Their scoring system formulation makes these disparities explicit and provides a framework for designing models that achieve equitable recourse availability as a design criterion rather than an afterthought. The credit domain is the primary testing ground for the representation-driven causal chain (Section~\ref{sec:mechanisms}), where data imbalance in historical loan records produces embedding skew that propagates through attribution distortion into recourse disparity.

\subsection{Employment and Hiring}

The employment domain presents explanation fairness challenges distinct from credit and criminal justice because the decision is inherently forward-looking (assessing potential) rather than backward-looking (assessing history), and because the feature space includes assessments of subjective qualities such as ``culture fit'' and ``communication style'' that are particularly susceptible to proxy leakage~\cite{dastin2018amazon}.

Amazon's internal resume screening algorithm, discontinued in 2018~\cite{dastin2018amazon}, systematically downgraded resumes containing the word ``women's'' (as in ``women's chess club'') and graduates of all-women's colleges, having learned from ten years of male-dominated hiring data that maleness correlated with success. This is a typical case of representation-driven inequity, where the training data encoded historical gender discrimination, producing embedding skew that attributed high importance to male-coded features. Critically, if explanation methods had been applied to explain the algorithm's decisions to reject candidates, the explanation would have attributed credit to ostensibly neutral features (technical skills, university rankings) while concealing the latent gender signal, a common Axiom 2 violation (counterfactual stability).

Raghavan et al.~\cite{raghavan2020mitigating} survey AI-based hiring tools and find that auditing practices are inconsistent and often insufficient, with post-hoc explanation auditing being particularly unreliable due to the identifiability barrier (Proposition~\ref{prop:identifiability}). The employment domain also raises unique intersectional concerns, which is the intersection of gender and race in hiring discrimination, is empirically documented to produce subgroup-specific explanation failures that are invisible to single-attribute auditing~\cite{buolamwini2018gender}.

\subsection{Cross-Domain Synthesis: Structural Features of High-Stakes Explanation Fairness}

Three structural features recur across all four domains and should be treated as design requirements for any explanation fairness framework deployed in high-stakes settings.

\emph{The actionability-feasibility gap is domain-specific.} Actionable recourse in credit means changing financial behavior; in criminal justice, it means changing social circumstances; in healthcare, it may be biologically impossible; in employment, it may require institutional access that is structurally unequal. A framework that treats actionability as domain-agnostic will systematically underestimate the recourse burden for minority groups in the most constrained domains.

\emph{Temporal dynamics differ by domain.} Credit scoring models face distributional shift driven by economic cycles; healthcare models face shift driven by treatment innovations and population health changes; employment models face shift driven by labor market conditions. The rate of explanation obsolescence (Section~\ref{sec:open}) therefore varies systematically by domain, and temporal fairness frameworks must be calibrated domain-specifically.

\emph{Regulatory context determines the acceptable epistemic access floor.} GDPR Article 22 establishes a right to explanation in all four domains, but the operationalization of ``meaningful explanation'' differs: CFPB guidance specifies adverse action notice requirements in credit; NHS guidance specifies patient information standards in healthcare; EEOC guidance specifies anti-discrimination disclosure in employment. Explanation fairness frameworks must be regulatory-aware to be deployable.

\section{Research Agenda: Imperatives and Open Problems}
\label{sec:open}
This section does two things that a definitive survey must do. First, we state a set of \emph{research imperatives}, concrete directives specifying what the field must build, not merely questions it should investigate. Second, we catalog open problems prioritized by urgency and tractability, distinguishing between problems ripe for solution with existing tools and those requiring conceptual breakthroughs.

\subsection{Research Imperatives}
\label{sec:imperatives}

Five imperatives follow directly from the analysis in this survey. They are stated as requirements, not suggestions.

\textbf{Imperative 1: Explanation-aware training must replace post-hoc explanation in high-stakes domains.} Proposition~\ref{prop:identifiability} establishes that post-hoc methods cannot certify procedural fairness. For systems that make decisions in credit, criminal justice, healthcare, and employment, this structural limitation is unacceptable. Training-time methods that enforce the conditional invariance condition (Definition~\ref{def:ef}) as a hard constraint, not a regularization term, must become the standard. This requires developing differentiable enforcement mechanisms such as GCIG and FairX, and integrating them with regulatory compliance frameworks.

\textbf{Imperative 2: Explanation fairness evaluation must be standardized in shared benchmarks.} Progress is currently unmeasurable because studies use inconsistent datasets, preprocessing choices, and metrics. A standardized benchmark analogous to GLUE for NLP, with canonical preprocessing, shared protected attribute definitions, and coverage of all explanation types (attribution, counterfactual, rule-based, graph), must be built. Without this, the field cannot accumulate evidence or compare methods.

\textbf{Imperative 3: Causal identification must be integrated into explanation generation.} The identifiability barrier (Section~\ref{sec:identifiability}) shows that observational data alone cannot certify explanation fairness. Methods that integrate causal discovery~\cite{spirtes2000causation} with explanation generation, producing explanations that are faithful to causal structure, not merely statistical patterns,must replace purely correlational attribution methods for deployments where causal accountability is legally or ethically required.

\textbf{Imperative 4: Human-centered evaluation of epistemic fairness must become an important audit requirement.} Axiom 5 (epistemic accessibility) is currently unmeasured by any quantitative metric. Regulatory bodies (CFPB, NHS, ECOA) must require evidence that explanations are equally comprehensible and actionable across demographic groups, not merely equal in information content. This requires developing validated human evaluation protocols, analogous to the GDPR-mandated Data Protection Impact Assessment, specifically for explanation fairness.

\textbf{Imperative 5: Temporal monitoring of explanation fairness must be built into deployment pipelines.} Explanation obsolescence is a known, measurable phenomenon. Deployers of ML systems in dynamic environments must implement continuous monitoring of $\text{TEF}(t)$ across protected groups, with automated alerts when the differential explanation fairness drift crosses a threshold, and mandatory recalibration protocols. This should be a regulatory requirement, not an optional best practice.

\subsection{Open Problems, Prioritized by Urgency and Tractability}
\label{sec:openproblems}

We prioritize seven open problems on a two-dimensional scale: \emph{urgency} (how soon are solutions needed given deployment pace?) and \emph{tractability} (how close are current methods to a solution?). High-urgency, high-tractability problems should attract immediate research investment; high-urgency, low-tractability problems need conceptual breakthroughs and longer-horizon funding.

\paragraph{P1 [Urgency: High | Tractability: High] Unified Evaluation Benchmarks.}
The absence of standardized benchmarks for explanation fairness is the single most immediately solvable obstacle to field progress. The technical work required is modest, including canonical preprocessing pipelines and shared evaluation protocols, but coordination across research groups is needed. A standardized benchmark suite covering feature attribution methods (SHAP, LIME, integrated gradients), counterfactual methods (DiCE, Wachter, Growing Spheres), rule-based methods (ANCHOR), and graph methods (GNNExplainer, REFEREE) could be assembled within one to two years.

\paragraph{P2 [Urgency: High | Tractability: Moderate] Explanation Fairness for LLMs.}
Three distinct fairness problems arise in LLMs: chain-of-thought disparities (reasoning chains exhibiting demographic bias independent of outputs), alignment-explanation mismatch (alignment training systematically suppressing accurate but sensitive explanations), and post-hoc rationalization (generated explanations disconnected from actual computation). Natural language metrics analogous to GPF$_\text{FAE}$ for generated text, faithfulness probing methods, and causal intervention approaches in LLM internals are all tractable research directions.

\paragraph{P3 [Urgency: High | Tractability: Moderate] Temporal Fairness Under Distributional Shift.}
The formal TEF(t) metric provides a measurement framework; what is needed is a set of shift-adaptive explanation methods that recalibrate without full model retraining, and longitudinal datasets to validate them. This is tractable with existing distributional robustness methodology.

\paragraph{P4 [Urgency: High | Tractability: Low] Causal Identification for Explanation Fairness.}
Integrating causal discovery with explanation generation requires solving the identifiability problem under realistic conditions: unknown causal graphs, hidden confounders, and non-linear causal mechanisms. No existing method achieves this; progress requires simultaneous advances in causal discovery theory and graph-based explanation methods.

\paragraph{P5 [Urgency: Moderate | Tractability: Moderate] Intersectional Explanation Fairness at Scale.} Though \cite{popoola2025mesd} proposed MESD to mitigate explanation bias in intersectional groups, the field of intersectional explanation fairness is still open, and the combinatorial explosion of intersectional subgroups makes comprehensive auditing computationally intractable in the worst case~\cite{kearns2018gerrymandering}. Efficient algorithms leveraging multicalibration~\cite{hebert2018multicalibration} structure or game-theoretic auditor-learner frameworks may provide tractable approximations. This is an active research area with concrete algorithmic entry points.

\paragraph{P6 [Urgency: Moderate | Tractability: Low] Formal Epistemic Fairness Metrics.}
Axiom 5 requires human evaluation protocols that cannot currently be automated. Developing psychometric instruments validated across demographic groups, calibrated to domain-specific regulatory requirements, and administrable at scale,without requiring individual human evaluation of every model decision,is a genuine conceptual challenge requiring interdisciplinary work between ML researchers, cognitive scientists, and regulatory experts.

\paragraph{P7 [Urgency: Low | Tractability: High] Human-Centered and Participatory Design Standards.}
The technical infrastructure for participatory explanation design exists; what is missing is standardized protocols for involving affected communities in defining explanation fairness criteria. Methodological work from social computing and participatory action research provides a foundation; adapting it to ML deployment contexts is tractable but requires sustained community partnerships.

\section{Conclusion}
\label{sec:conclusion}

This survey has provided what we believe is the earliest theoretically unified treatment of the emerging field of fairness of explanations in artificial intelligence. Beyond surveying numerous publications and introducing a seven-dimensional taxonomy, we have made eight contributions that distinguish this work from a literature map and from prior surveys that treat explanation fairness as reducible to either outcome fairness or explanation quality.

\textbf{(1) Conditional invariance framework.} We formalized explanation fairness as a conditional invariance principle (Definition~\ref{def:ef}, Eq.~\ref{eq:ef_invariance}): an explanation method is fair if and only if its output distribution is invariant to the protected attribute when task-relevant features are held constant. This single principle unifies all existing metrics as partial operationalizations under different assumptions (Table~\ref{tab:unification}).

\textbf{(2) Five axioms of a fair explanation system.} We extracted five positive design axioms from the invariance condition: Process Consistency, Counterfactual Stability, Distributional Parity, Actionability Symmetry, and Epistemic Accessibility (Section~\ref{sec:axioms}). These axioms are not all mutually redundant; satisfying Axioms 1--4 does not imply Axiom 5, and none is currently satisfied by any post-hoc method by design.

\textbf{(3) Identifiability proposition and hard limit.} We formalized as Proposition~\ref{prop:identifiability} the structural result that no post-hoc method can certify explanation fairness from model outputs alone, and stated as Remark~\ref{rem:hardlimit} that post-hoc methods are \emph{structurally incapable} of certifying procedural fairness. Table~\ref{tab:posthoc_vs_intrinsic} makes the architectural implications concrete across three design families.

\textbf{(4) Three generative mechanisms with full causal chain.} We identified three pathways through which explanation inequity arises (Section~\ref{sec:mechanisms}): representation-driven inequity, explanation-model mismatch, and actionability-driven inequity. Crucially, we traced the full causal chain explicitly, from data imbalance through latent embedding skew and decision boundary asymmetry to attribution distortion and recourse disparity, making clear that debiasing any single link leaves the others intact.

\textbf{(5) Standard six-step evaluation workflow.} We proposed a canonical audit protocol (Section~\ref{sec:eval_workflow}) designed to probe all three generative pathways in sequence, from automated metric-based steps through human interpretability validation. This is designed to become the standard audit protocol for deployed systems.

\textbf{(6) Explicit statement of what explanation fairness cannot guarantee.} We stated five hard limits without hedging (Section~\ref{sec:limits}): post-hoc methods cannot certify explanation fairness; explanation fairness is not equivalent to model fairness; post-hoc methods cannot ensure causal neutrality; epistemic inequities are invisible to statistical metrics; the identifiability barrier is permanent. A field cannot mature until it knows its own limits.

\textbf{(8) Research agenda: five imperatives and seven prioritized open problems.} Unlike open-problems lists that catalog questions, we stated five concrete directives specifying what must be built (Section~\ref{sec:imperatives}): explanation-aware training must replace post-hoc auditing in high-stakes domains; benchmarks must be standardized; causal identification must be integrated into explanation generation; epistemic fairness must become a first-class regulatory requirement; temporal monitoring must be built into deployment pipelines. Seven open problems are then prioritized by urgency and tractability, enabling strategic research investment.

Two themes run through all eight contributions. First, most of the research in the field has been approaching the problem of explanation fairness wrongly: auditing explanations post-hoc cannot substitute for enforcing fairness during training and architecture design. Proposition~\ref{prop:identifiability} makes this not a recommendation but a structural necessity. Second, the procedural gap, the gap between outcome fairness and explanation fairness, is not a minor measurement refinement, but a fundamental conceptual gap that prior works in both fairness and XAI have systematically missed.

The practical stakes are substantial. As XAI methods become standard components of regulated AI systems across credit, healthcare, criminal justice, and employment, the equity of explanations will face the same legal and social scrutiny as the equity of predictions. The framework, axioms, propositions, mechanism taxonomy, evaluation workflow, failure taxonomy, imperatives, and explicit limits provided here offer a structured theoretical foundation for meeting that challenge. This survey is intended not as a snapshot of the literature but as the reference against which future work in explanation fairness measures itself.

%%
%% The next line prints the references.
\printbibliography

\end{document}